\documentclass{article}

\usepackage[preprint,nonatbib]{sty/neurips_2019}

\usepackage[utf8]{inputenc} 
\usepackage[T1]{fontenc}    
\usepackage{hyperref}       
\usepackage{url}            
\usepackage{booktabs}       
\usepackage{amsfonts}       
\usepackage{mathtools}
\usepackage{algorithm}
\usepackage{algorithmic}
\usepackage{nicefrac}       
\usepackage{microtype}      
\usepackage{thm-restate}

\usepackage{mathabx}
\usepackage[italic,defaultmathsizes]{mathastext}
\usepackage{upgreek}
\usepackage{amsthm}

\usepackage{color}
\usepackage{bm}
\usepackage{subfig}
\usepackage{multicol,multirow}
\usepackage{array}
\usepackage{frame}
\usepackage{adjustbox}
\usepackage[percent]{overpic}

\usepackage{wrapfig}
\usepackage{xfrac}
\usepackage[capitalise]{cleveref}
\usepackage{soul}
\setstcolor{magenta}

\newtheorem{lemma}{Lemma}
\newtheorem{corollary}{Corollary}
\newtheorem{theorem}{Theorem}

\newtheorem{proposition}{Proposition}
\newtheorem{remark}{Remark}

\theoremstyle{definition}
\newtheorem{definition}{Definition}

\def\one{\bm{1}}

\newcommand{\arcsinh}{\textnormal{arcsinh}}
\newcommand{\argmax}{\textnormal{argmax}}

\newcommand{\Blue}[1]{{\color{blue}  {#1}}}
\newcommand{\dom}{\mathcal{D}}
\newcommand{\R}{\mathbb{R}}
\newcommand{\EE}{\mathbb{E}}
\newcommand{\LL}{\mathrm{L}}

\newcommand{\ah}{{\hat{\bm{a}}}}
\newcommand{\ub}{\bm{u}}
\newcommand{\vb}{\bm{v}}
\newcommand{\lb}{\ell}
\def\Fs{\check{F}}
\def\fs{\check{f}}
\newcommand{\Fsub}{\Fs^*}
\def\Ft{\widetilde{F}}
\def\ft{\tilde{f}}
\newcommand{\w}{\bm{w}}
\newcommand{\wt}{\tilde{\w}}
\newcommand{\x}{\bm{x}}
\newcommand{\z}{\bm{z}}
\renewcommand{\a}{\bm{a}}
\newcommand{\y}{\bm{y}}
\newcommand{\yh}{\hat{\bm{y}}}

\newcommand{\Puk}{P_{\text{\tiny UK}}}
\newcommand{\Puksub}{P_{\text{\scalebox{0.9}{UK}}}}
\newcommand{\Dt}{D_{\text{\tiny Tsallis}}}
\newcommand{\loss}{L_{t_1}^{t_2}}
\newcommand{\tf}{s}
\newcommand{\Set}{\mathcal{S}}

\newcommand{\nclass}{\ensuremath k}

\title{Robust Bi-Tempered Logistic Loss \\ Based on Bregman Divergences}

\author{%
  Ehsan Amid$^{\,\star\dagger}$
  \quad
  Manfred K.\ Warmuth$^{\,\star\dagger}$
  \quad
  Rohan Anil$^{\,\dagger}$
  \quad
  Tomer Koren$^{\,\dagger}$\\\\
  $\star$\,Department of Computer Science, University of California, Santa Cruz\\
  $\dagger$\,Google Brain\\
  \texttt{\{eamid, manfred, rohananil, tkoren\}@google.com}
}

\begin{document}
\maketitle

\vspace{-0.1cm}
\begin{abstract}
We introduce a temperature into the exponential function and replace the softmax output layer of neural nets by a high temperature generalization. Similarly, the logarithm in the log loss we use for training is replaced by a low temperature logarithm. By tuning the two temperatures we create loss functions that are non-convex already in the single layer case. When replacing the last layer of the neural nets by our bi-temperature generalization of logistic loss, the training becomes more robust to noise. We visualize the effect of tuning the two temperatures in a simple setting and show the efficacy of our method on large data sets. Our methodology is based on Bregman divergences and is superior to a related two-temperature method using the Tsallis divergence. 
\end{abstract}

\vspace{-0.4cm}
\section{Introduction} \label{sec:intro}

The logistic loss, also known as the softmax loss, has been the standard choice in training deep neural networks for classification. The loss involves the application of the softmax function on the activations of the last layer to form the class probabilities followed by the relative entropy (aka the Kullback-Leibler (KL) divergence) between the true labels and the predicted probabilities. The logistic loss is known to be a convex function of the activations (and consequently, the weights) of the last layer. 

Although desirable from an optimization standpoint, convex
losses have been shown to be prone to outliers~\cite{long}
as the loss of each individual example unboundedly increases as a function of the activations. 
These outliers may correspond to extreme examples that lead to large gradients, or misclassified training examples that are located far away from the classification boundary. 
Requiring a convex loss function at the output layer thus
seems somewhat arbitrary, in particular since convexity in
the last layer's activations does not guarantee convexity
with respect to the parameters of the network outside the last layer.
Another issue arises due to the exponentially decaying tail of the softmax function that assigns probabilities to the classes. 
In the presence of mislabeled training examples near the classification boundary, the short tail of the softmax probabilities enforces the classifier to closely follow the noisy training examples.
In contrast, heavy-tailed alternatives for the softmax probabilities have been shown to significantly improve the robustness of the loss to these examples~\cite{tlogistic}. 

The logistic loss is essentially the logarithm of the predicted class probabilities, which are computed as the normalized exponentials of the inputs.
In this paper, we tackle both shortcomings of the logistic loss, pertaining to its convexity as well as its tail-lightness, by replacing the logarithm and exponential functions with their corresponding ``tempered'' versions. 
We define the function $\log_t: \R_+ \rightarrow \R$ 
with \emph{temperature} parameter $t \geq 0$ as in~\cite{texp1}:
\begin{equation}
\label{eq:logt}
\log_t(x) \coloneqq \frac{1}{1-t} (x^{1-t} - 1)\, .
\end{equation}
The $\log_t$ function is monotonically increasing 
and concave. The standard (natural) logarithm is recovered at the limit $t \rightarrow 1$. 
Unlike the standard $\log$, the $\log_t$ function is bounded from below by $-1/(1-t)$ for $0 \leq t < 1$. 
This property will be used to define bounded loss functions
that are significantly more robust to outliers. 
Similarly, our heavy-tailed alternative for the softmax function is based on the tempered exponential function. 
The function $\exp_t: \R \rightarrow \R_+$ with temperature $t \in \R$ is defined as the inverse\footnote{When $0 \leq t < 1$, the domain of $\exp_t$ needs to be restricted to $-1/(1-t) \leq x$ for the inverse property to hold.} of $\log_t$, that is,
\begin{equation} \label{eq:expt}
\exp_t(x) \coloneqq [1 + (1-t)\,x]_+^{1/(1-t)}\,\, ,
\end{equation}
where $[\,\cdot\,]_+ = \max\{\,\cdot\,, 0\}$. 
The standard $\exp$ function is again recovered at the limit $t\rightarrow 1$. 
Compared to the $\exp$ function, a heavier tail (for negative values of $x$) is achieved for $t>1$. We use this property to define heavy-tailed analogues of softmax probabilities at the output layer.

\begin{figure*}[t!]
\vspace{-1cm}
\begin{center}
\subfloat[]{\includegraphics[height=0.24\textwidth]{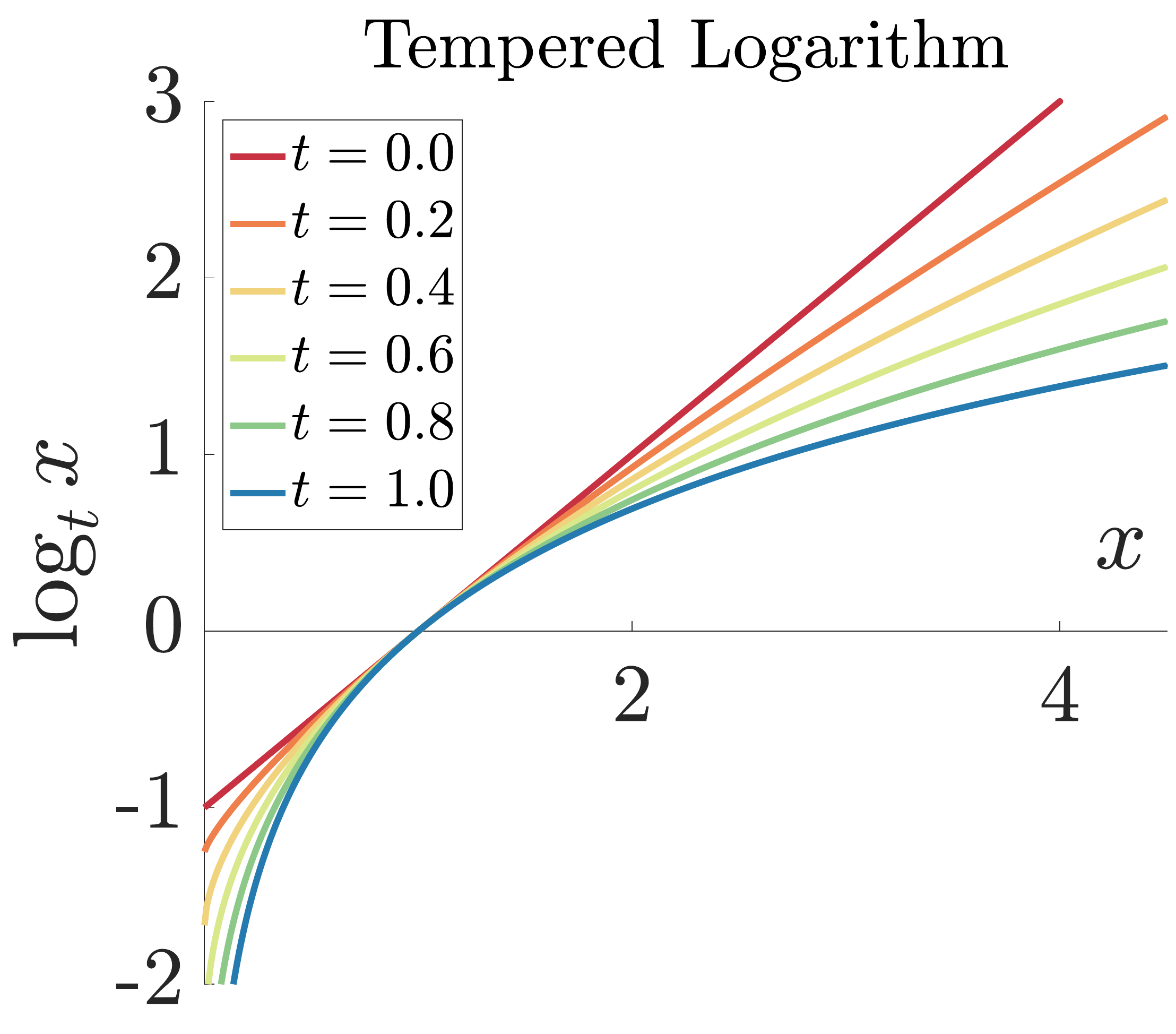}\label{fig:logt}}
\subfloat[]{\includegraphics[height=0.24\textwidth]{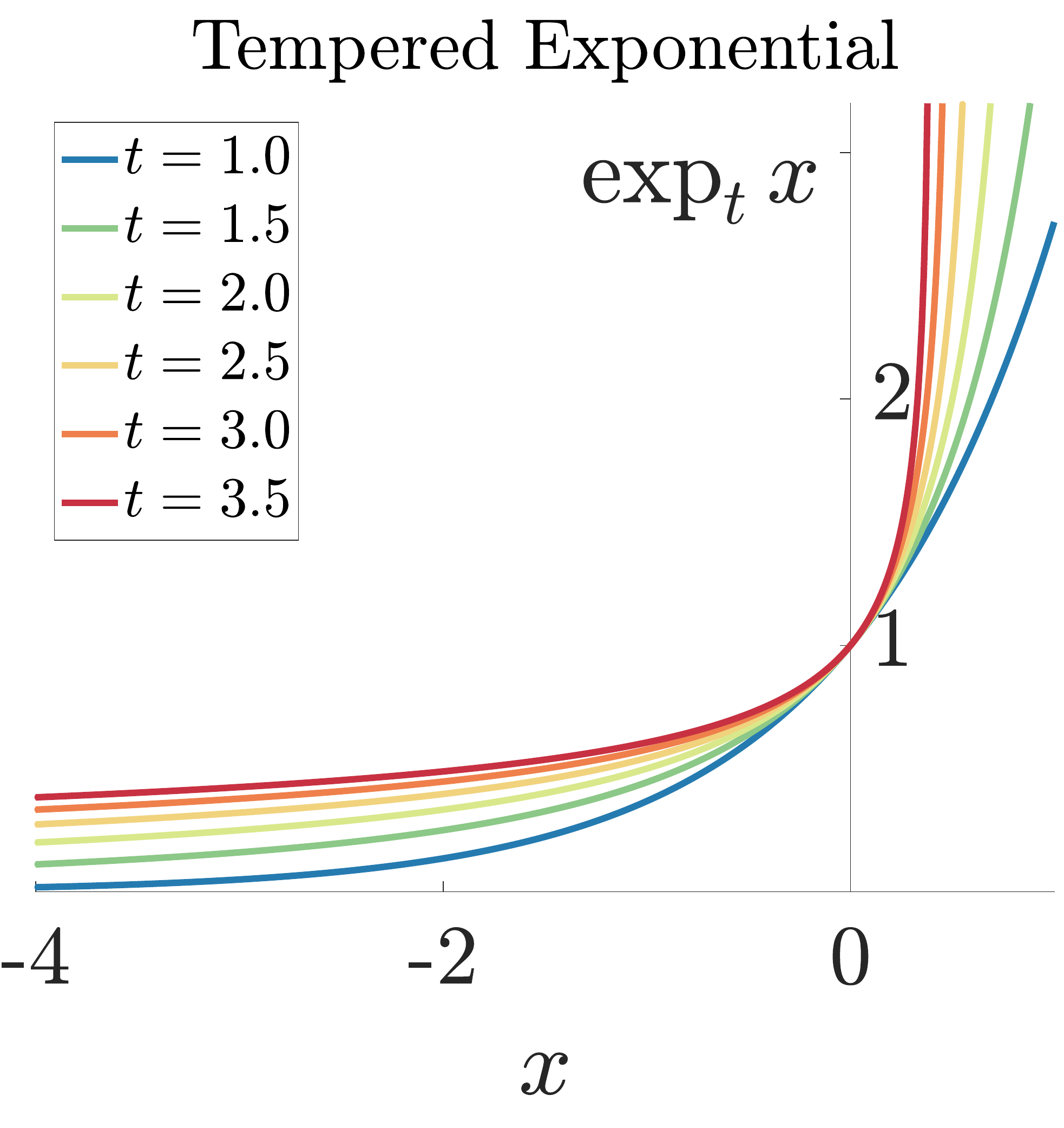}\label{fig:expt}}
\subfloat[]{\includegraphics[height=0.24\textwidth]{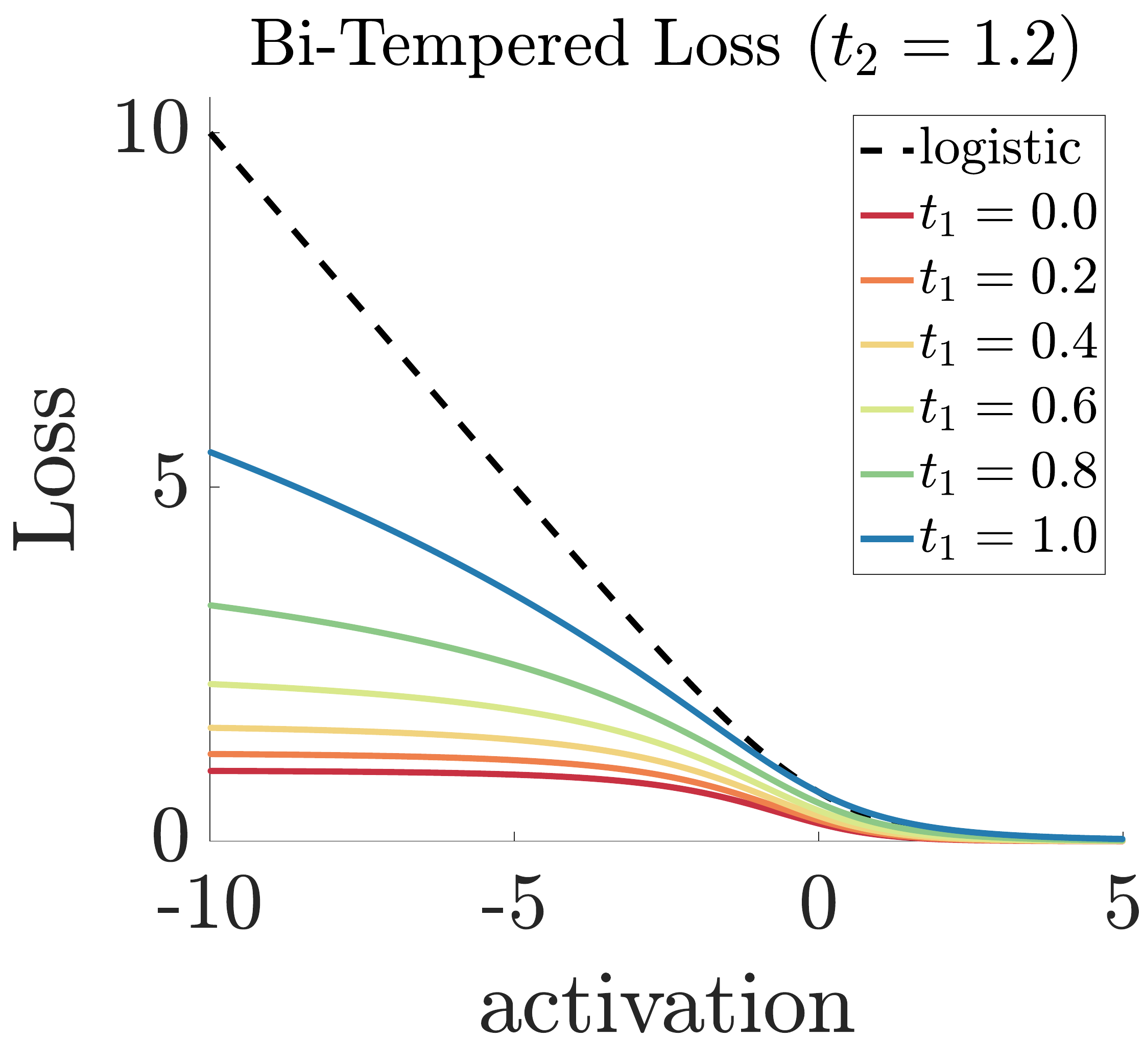}\label{fig:loss_t1}}
\subfloat[]{\includegraphics[height=0.24\textwidth]{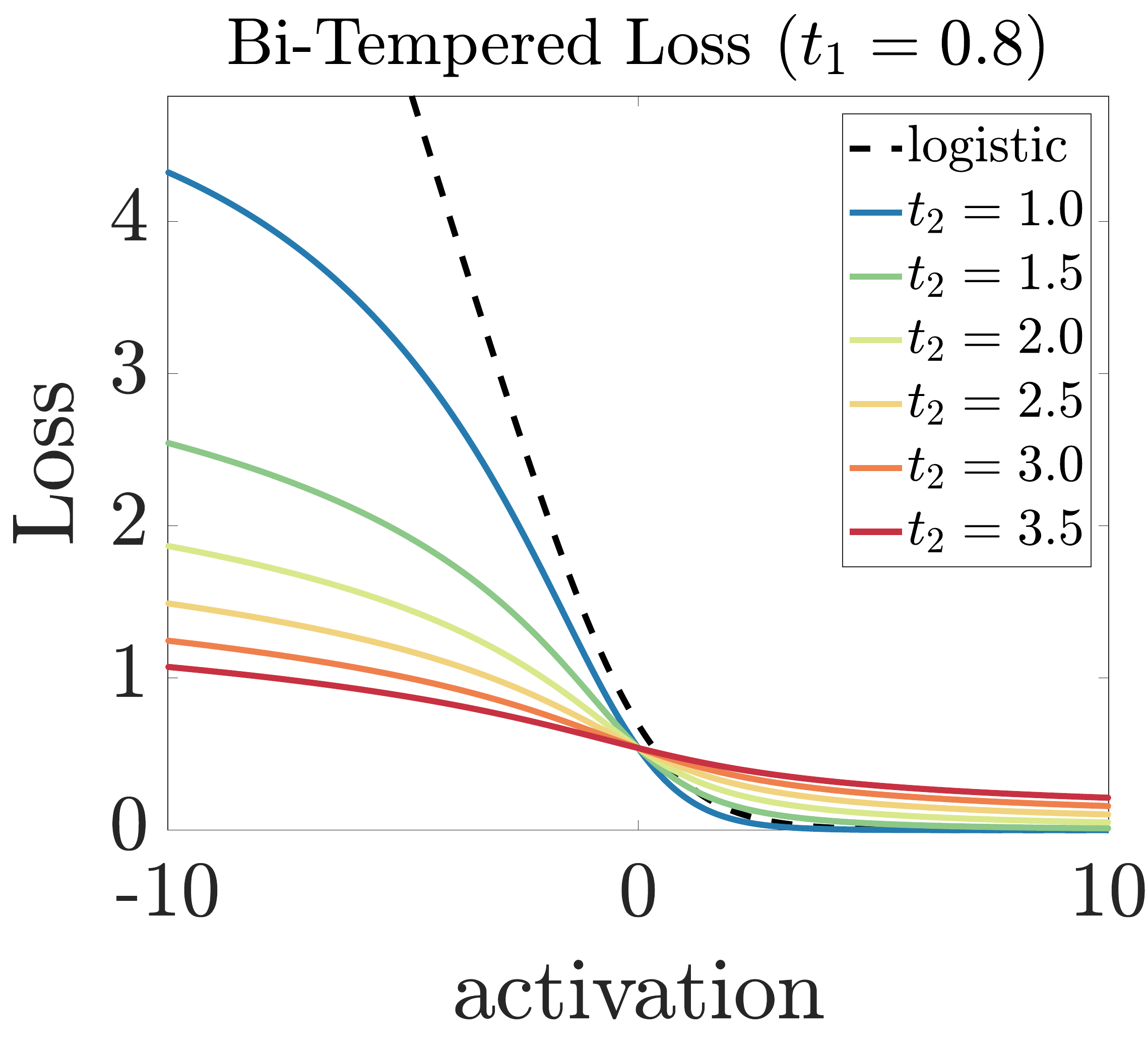}\label{fig:loss_t2}}\hfill
     \caption{Tempered logarithm and exponential functions, and the bi-tempered logistic
     loss: (a) $\log_t$ function, (b) $\exp_t$ function,
     bi-tempered logistic loss when (c) $t_2 = 1.2$ fixed and $t_1 \leq 1$, and (d) $t_1 = 0.8$ fixed and $t_2 \geq 1$.}\label{fig:pre}
\end{center}
\vspace{-0.5cm}
\end{figure*}

The vanilla logistic loss can be viewed as a logarithmic
(relative entropy) divergence that operates on a ``matching'' exponential (softmax) probability assignment \cite{match,matchmulti}.
Its convexity then stems from classical convex duality, using the fact that the probability assignment function is the gradient of the dual function to the entropy on the simplex.
When the $\log_{t_1}$ and $\exp_{t_2}$ are substituted
instead, this duality still holds whenever $t_1 = t_2$,
albeit with a different Bregman divergence, and the induced loss remains convex%
\footnote{In a restricted domain when $t_1 = t_2 < 1$, as discussed later.}.
However, for $t_1 < t_2$, the loss becomes non-convex in the output activations. 
In particular, $0 \leq t_1 < 1$ leads to a bounded loss, while $t_2 > 1$ provides tail-heaviness. 
Figure~\ref{fig:pre} illustrates the tempered $\log_t$ and $\exp_t$ functions as well as examples of our proposed bi-tempered logistic loss function for a $2$-class problem expressed as a function of the activation of the first class. 
The true label is assumed to be class one.

Tempered generalizations of the logistic regression have been introduced 
before~\cite{ding,tlogistic,zhang2018,ourtsallis}. 
The most recent two-temperature method~\cite{ourtsallis} 
is based on the Tsallis divergence and
contains all the previous methods as special cases. 
However,
the Tsallis based divergences do not result in proper loss functions. 
In contrast, we show that the Bregman based construction introduced in this paper 
is indeed proper, which is a requirement for many real-world  applications.

\subsection{Our replacement of the softmax output layer in neural nets}

Consider an arbitrary classification model with multiclass softmax output.
We are given training examples of the form $(\x,\y)$, where $\x$ is
a fixed dimensional input vector and the target $\y$ is
a probability vector over $\nclass$ classes. 
In practice, the targets are often one-hot encoded binary vectors in $\nclass$ dimensions. 
Each input $\x$ is fed to the model, resulting in a vector $\z$ of inputs to the output softmax. 
The softmax layer has typically one trainable weight vector $\w_i$ per class $i$ and yields the predicted class probability
$$
\hat{y}_{i} = \frac{\exp(\hat{a}_{i})}{\sum_{j=1}^k \exp(\hat{a}_{j})} 
= \exp \Big(\hat{a}_{i}-\log\sum_{j=1}^k \exp(\hat{a}_{j})\Big), 
\text{ for linear activation }\hat{a}_{i} = \w_i \cdot \z 
\text{ for class $i$.}
$$
We first replace the softmax function by a generalized heavy-tailed version that uses the $\exp_{t_2}$ function with $t_2 > 1$, which we call the \emph{tempered softmax function}:
$$
\hat{y}_{i} =\exp_{t_2}\big(\hat{a}_{i}-\lambda_{t_2}(\ah)\big),
\text{\,\, where\,\, } 
\lambda_{t_2}(\ah)\in\R \text{\,\, is s.t. \,}
\sum_{j=1}^k \exp_{t_2}\big(\hat{a}_{j}-\lambda_{t_2}(\ah)\big) = 1\, .
$$
This requires computing the normalization value $\lambda_{t_2}(\ah)$ (for each example) via a binary search or an iterative procedure like the one given in Appendix~\ref{a:alg}.
The relative entropy between the true label $\y$ and prediction $\yh$ is 
replaced by the tempered version with temperature $0 \leq t_1 < 1$,
\[
    \sum_{i=1}^{\nclass} \big(y_{i}\,  (\log_{t_1} y_{i} - \log_{t_1} \hat{y}_{i})
	- \tfrac{1}{2-t_1}\, (y_{i}^{2-t_1} - \hat{y}_{i}^{2-t_1}) \big)
 	\;\;\stackrel{\text{if }\y\, \text{one-hot}}=\;\;
 	  - \log_{t_1} \hat{y}_{c} - \tfrac{1}{2-t_1}
 	  \Big( 1-\sum_{i=1}^{\nclass} \hat{y}_{i}^{2-t_1} \Big)\, .
\]
 where $c = \argmax_i\, y_i$ is the index of the one-hot
 class. We motivate this loss in later sections.
When $t_1=t_2=1$, then it reduces to the vanilla logistic loss for the softmax.
On the other hand, when $0 \leq t_1 < 1$, then the loss is bounded, while $t_2 > 1$ gives the tempered softmax function a heavier tail. 

\begin{figure*}[t]
\vspace{-1cm}
\hspace{-2.3cm}
\footnotesize
\begin{tabular}{m{0.1\textwidth} m{0.24\textwidth} m{0.24\textwidth} m{0.24\textwidth} m{0.24\textwidth}}
\multicolumn{1}{c}{\rotatebox{90}{Logistic}}\hspace{-1.4cm} &
\includegraphics[height=0.24\textwidth]{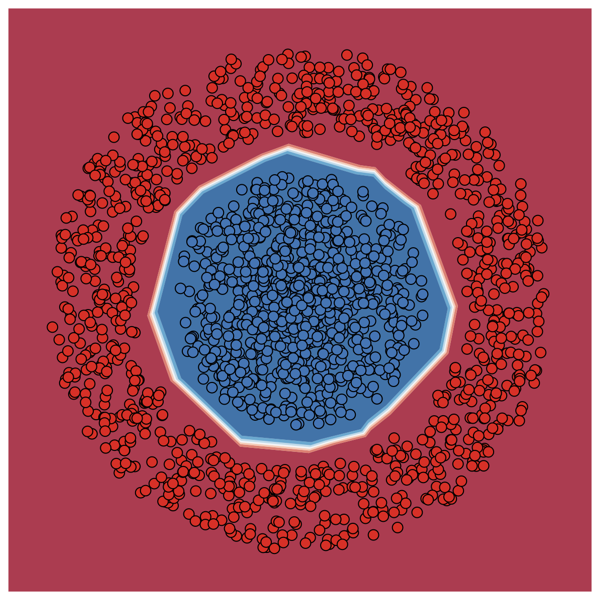} &
\includegraphics[height=0.24\textwidth]{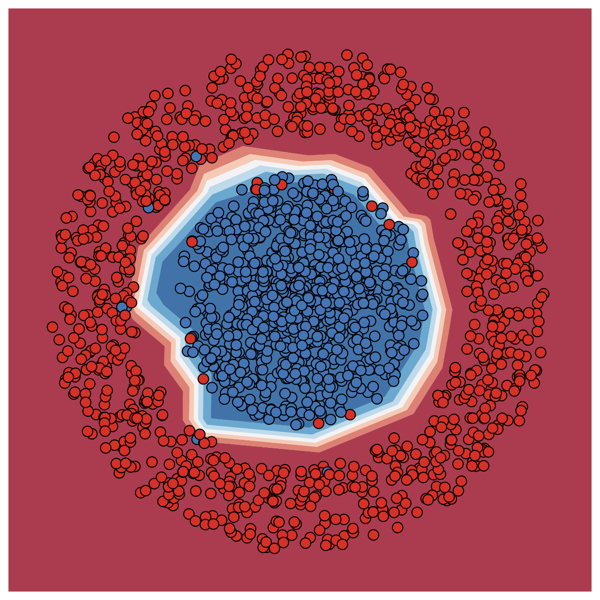} &
\includegraphics[height=0.24\textwidth]{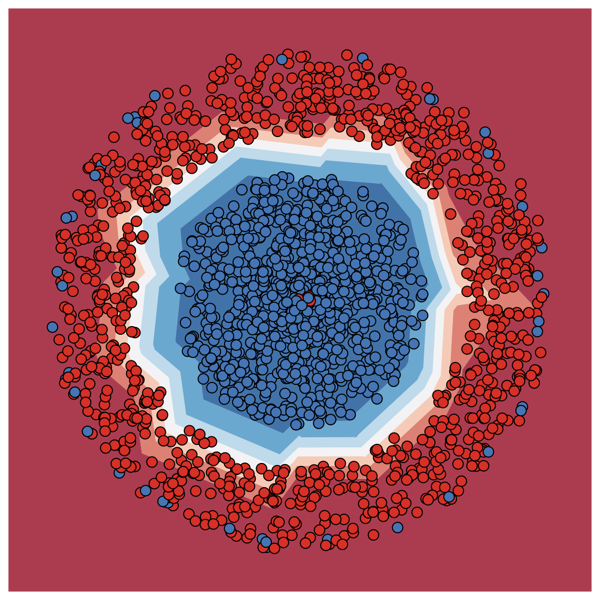} &
\includegraphics[height=0.24\textwidth]{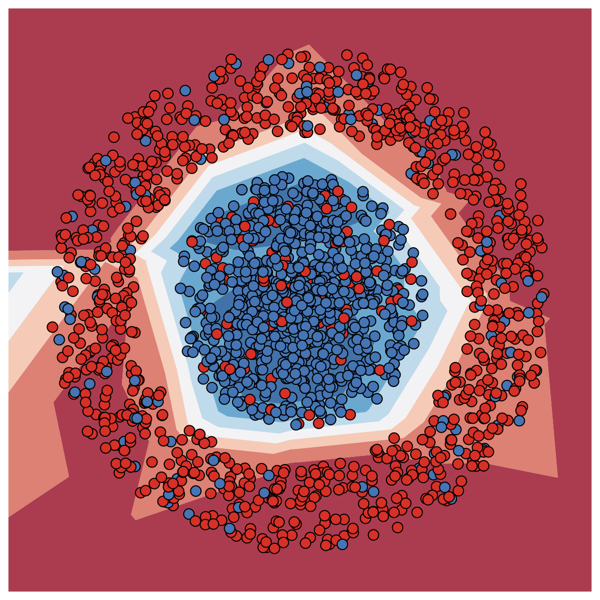}\hfill\\[2mm]
  &\multicolumn{1}{c}
   {\footnotesize $\begin{array}{c}
                      \text{bounded \& heavy-tail}\\
                      (0.2,4.0)
                   \end{array}$} 
  & \multicolumn{1}{c}
    {\footnotesize $\begin{array}{c}
	               \text{only heavy-tail}\\
	               (1.0, 4.0)
                    \end{array}$}
    &  \multicolumn{1}{c}
    {\footnotesize $\begin{array}{c}
	               \text{only bounded}\\
	                (0.2, 1.0)
                    \end{array}$} 
    & \multicolumn{1}{c}
    {\footnotesize $\begin{array}{c}
                	\text{bounded \& heavy-tail}\\
	                 (0.2, 4.0)
                    \end{array}$}\\
\multicolumn{1}{c}{\rotatebox[origin=c]{90}{Bi-Tempered}}\hspace{-1.4cm} &
\includegraphics[height=0.24\textwidth]{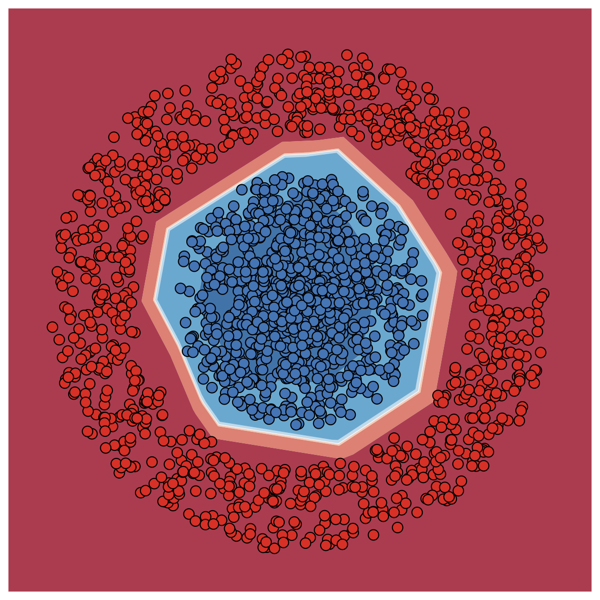} &
\includegraphics[height=0.24\textwidth]{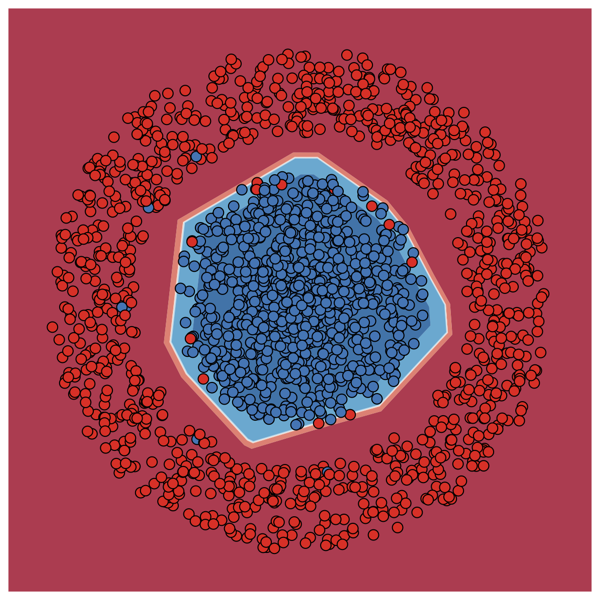} &
\includegraphics[height=0.24\textwidth]{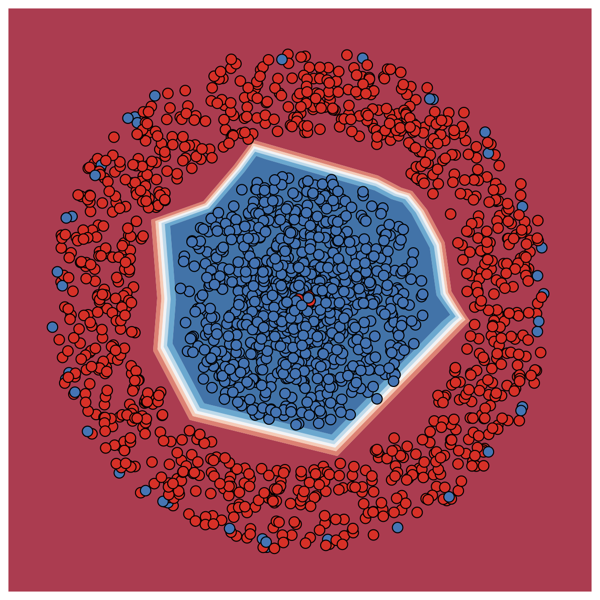} &
\includegraphics[height=0.24\textwidth]{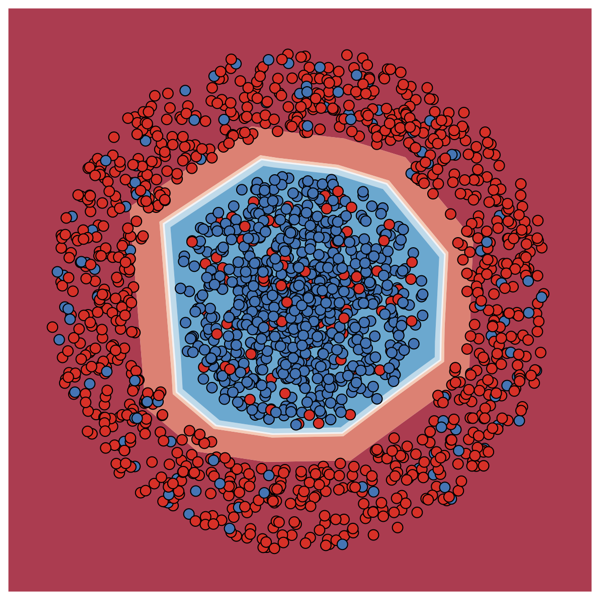}\hfill\\
& \multicolumn{1}{c}{{\small (a)}} & \multicolumn{1}{c}{{\small (b)}} &  \multicolumn{1}{c}{{\small (c)}} & \multicolumn{1}{c}{{\small (d)}}\\
\end{tabular}
\caption{Logistic vs. robust bi-tempered logistic loss: (a) noise-free labels, (b) small-margin label noise, (c) large-margin label noise, and (d) random label noise. The temperature values $(t_1,t_2)$ for the bi-tempered loss are shown above each figure.
    }\label{fig:example}
    \vspace{-5mm}
\end{figure*}

\subsection{An illustration}

We provide some intuition on why both boundedness of the loss as well as tail-heaviness of the tempered softmax are crucial for robustness. 
For this, we train a small two layer feed-forward neural network
on a synthetic binary classification problem in two dimensions.
The network has $10$ and $5$ units in the first and second layer, respectively. 
Figure~\ref{fig:example}(a) shows the results of the logistic and our bi-tempered logistic loss on the noise-free dataset. The network converges to a desirable classification boundary (the white stripe in the figure) using both loss functions. In Figure~\ref{fig:example}(b), we illustrate the effect of adding small-margin label noise to the training examples, targeting those examples that reside near the noise-free classification boundary. The logistic loss clearly follows the noisy examples by stretching the classification boundary. 
On the other hand, using \emph{only} the tail-heavy tempered softmax function ($t_2 = 4$ while $t_1 = 1$, i.e. KL divergence as the divergence) 
can handle the noisy examples by producing more uniform class probabilities. Next, we show the effect of
large-margin noisy examples in Figure~\ref{fig:example}(c), targeting examples that are located far away from the noise-free classification boundary.
The convexity of the logistic loss causes the network to be
highly affected by the noisy examples that are located far
away from the boundary. In contrast, \emph{only} the
boundedness of the loss ($t_1 = 0.2$ while $t_2=1$, meaning
that the outputs are vanilla softmax probabilities) reduces
the effect of the outliers by allocating at most a finite
amount of loss to each example. Finally, we show the effect
of random label noise that includes both small-margin and
large-margin noisy examples in Figure~\ref{fig:example}(d).
Clearly, the logistic loss fails to handle the noise,
while our bi-tempered logistic loss successfully recovers the
appropriate boundary. Note that to handle the random noise,
we exploit \emph{both} boundedness of the loss ($t_1 = 0.2< 1$) 
as well as the tail-heaviness of the probability assignments ($t_2 = 4 > 1$).

The theoretical background as well as our treatment of the softmax layer of the neural networks are developed in later sections.
In particular, we show that special discrete choices of the
temperatures result in a large variety of divergences
commonly used in machine learning. As we show in our experiments,
tuning the two temperatures as continuous
parameters is crucial. 

\vspace{-0.1cm}
\subsection{Summary of the experiments}
We perform experiments by adding synthetic label noise to MNIST and CIFAR-100 datasets and compare the results of our robust bi-tempered loss to the vanilla logistic loss. Our bi-tempered loss is significantly more robust to label noise; it provides $98.56\%$ and $62.55\%$ accuracy on MNIST and CIFAR-100, respectively, when trained with $40\%$ label noise (compared to $97.64\%$ and $53.17\%$, respectively, obtained using logistic loss). The bi-tempered loss also yields improvement over the state-of-the-art results on the ImageNet-2012 dataset using both the Resnet18 and Resnet50 architectures (see Table~\ref{tab:imagenet_results}).
\vspace{-0.2cm}
\section{Preliminaries}
\label{sec:pre}
\vspace{-0.1cm}
\subsection{Convex duality and Bregman divergences on the simplex}
\label{sec:bregman}

We start by briefly reviewing some basic background in convex analysis. 
For a continuously-differentiable strictly convex function
$F:\, \dom \rightarrow \R$, with convex domain $\dom$,
the Bregman divergence between $\y, \yh \in \dom$ induced by $F$ is defined as 
\[ \Delta_F(\y, \yh) = F(\y) - F(\yh) - (\y - \yh)\cdot f(\yh) \ ,\]
where $f(\yh) \coloneqq \nabla F(\yh)$ denotes the 
gradient of $F$ at $\yh$ (sometimes called the link function of $F$).
Clearly $\Delta_F(\y, \yh) \geq 0$  
and $\Delta_F(\y, \yh) = 0$ iff $\y = \yh$.
Also the Bregman divergence is always convex in the first argument and
$\nabla_{\y} \;\Delta_F(\y, \yh) =f(\y)-f(\yh)$, 
but not generally in its second argument. 

Bregman divergence generalizes many well-known divergences
such as the squared Euclidean
 $\Delta_F(\y, \yh) = \frac{1}{2}\,\Vert\y - \yh\Vert^2_2$ 
(with $F(\y) = \frac{1}{2}\,\Vert\y\Vert^2_2$)
and the Kullback-Leilbler (KL) divergence 
$\Delta_F(\y, \yh) = \sum_i (y_i \log\frac{y_i}{\hat{y}_i} - y_i + \hat{y}_i)$
(with $F(\y) = \sum_i (y_i \log y_i - y_i)$). 
Note that the Bregman divergence is not symmetric in general,
i.e., $\Delta_F(\y, \yh) \neq \Delta_F(\yh, \y)$. 
Additionally, the Bregman divergence is invariant to adding affine functions to 
the convex function $F$:\, $\Delta_{F + A}(\y, \yh) =
\Delta_F(\y, \yh)$, where $A(\y) = b + \bm{c}\cdot \y$ for arbitrary $b \in \R, \, \bm{c} \in \R^k$.

For every differentiable strictly convex function $F$ 
(with domain $\dom \subseteq \R_+^\nclass$), there exists a convex dual $F^*:\dom^* \rightarrow \R$ function such that for dual parameter pairs $(\y, \a)$, \, $\a \in \dom^*$, the following holds: 
$\a = f(\y)$ and $\y = f^*(\a) =\nabla F^*(\a) = f^{-1}(\a)$. However, we are mainly interested in the dual of the function $F$ when the domain is restricted to the probability simplex $S^k \coloneqq \{\y \in \R_+^k\vert\, \sum^k_{i=1} y_i = 1\}$. Let $\Fs^*:\check{\dom}^* \rightarrow \R$ denote the convex conjugate of the restricted function $F:\, \dom \cap S^k \rightarrow \R$,
\[
    \Fs^*(\a) 
    = \sup_{\y' \in \dom \cap S^k} \big(\y'\cdot\a - F(\y')\big) 
    = \sup_{\y' \in \dom} \inf_{\lambda\in\R\vphantom{^k}} \big(\y'\cdot\a - F(\y') + \lambda\, (1 - \sum_{i=1}^\nclass y'_i)\big)\, ,
\]
where we introduced a Lagrange multiplier $\lambda \in \R$ to enforce the constraint $\y' \in S^k$. 
At the optimum, the following relationships hold between the primal and dual variables:
\begin{equation}
\label{eq:links}
f(\y) = \a - \lambda(\a)\,\one\,\,\,\text{ and }\,\,\, \y = f^{-1}\big(\a - \lambda(\a)\,\one\big) = \fs^*(\a)\,,
\end{equation}
where $\lambda(\a)$ is chosen so that $\y \in S^k$. 
Note the dependence of the optimum $\lambda$ on $\a$.


\vspace{-0.1cm}
\subsection{Matching losses}

Next, we recall the notion of a \emph{matching loss}~\cite{match,matchmulti,buja,reidwill}.
It arises as a natural way of defining a loss function over
activations $\ah \in \R^\nclass$, by first mapping them to
a probability distribution using a \emph{transfer function}
$\tf : \R^\nclass \to S^\nclass$ that assigns probabilities to classes, and then computing a \emph{divergence} $\Delta_F$ between this distribution and the correct target labels.
The idea behind the following definition is to match the
transfer function and the divergence via duality.

\begin{definition}[Matching Loss]
Let $F:\, S^\nclass \rightarrow \R$ a continuously-differentiable, strictly convex function and let $\tf:\, \R^k \rightarrow S^k$ be a transfer function 
such that $\yh = \tf(\ah)$ denotes the predicted probability distribution
based on the activations $\ah$. Then the loss function
\[
L_F(\ah \mid \y ) \coloneqq \Delta_F(\y, \tf(\ah))\, ,
\]
is called the \emph{matching loss} for $\tf$, if $\tf = \fs^* = \nabla \Fs^*$.
\end{definition}

This matching is useful due to the following property.

\begin{proposition}
\label{lem:convexity}
The matching loss $L_F(\ah \mid \y)$ is convex w.r.t.~the activations $\ah\, \in \text{\textbf{\emph{range}}}((\fs^*)^{-1})$.
\end{proposition}

\begin{proof}
Note that $\Fs^*$ is a strictly convex function and
the following relation holds between the 
divergences induced by $F$ and $\Fs^*$:
\begin{equation}
\label{eq:dual-div} \Delta_F\big(\y, \yh\big) =
    \Delta_{\Fsub}\big((\fs^*)^{-1}(\yh),\,(\fs^*)^{-1}(\y)\big).
\end{equation}
Thus for any $\ah$ in the range of $(\fs^*)^{-1}$,
\[ \Delta_F\big(\y, \fs^*(\ah)\big) = \Delta_{\Fsub}\big(\ah, (\fs^*)^{-1}(\y)\big) .\]
The claim now follows from the convexity of the Bregman divergence $\Delta_{\Fsub}$ w.r.t.~its first argument.
\end{proof}

The original motivating example for the matching loss was
the logistic loss \cite{match,matchmulti}. 
It can be obtained as the matching loss for the softmax function
\[
\hat{y}_i = [\fs^*(\ah)]_i =
\frac{\exp(\hat{a}_i)}{\sum_{j=1}^{\nclass} \exp(\hat{a}_j)},
\]
which corresponds to the relative entropy (KL) divergence 
\[
L_F(\ah \mid \y) = 
\Delta_F\big(\y, \fs^*(\ah)\big) = 
    \sum_{i=1}^{\nclass} y_i \, (\log y_i - \log \hat{y}_i) = \sum_{i=1}^{\nclass} \big(y_i\,\log y_i - y_i\,\hat{a}_i)\big) + \log \big(\sum_{i=1}^\nclass \exp(\hat{a}_i)\big)\, ,
\]
induced from the negative entropy function 
$F(\y) = \sum_{i=1}^{\nclass} (y_i \log y_i - y_i)$.
We next define a family of convex functions $F_t$ parameterized by a temperature $t \geq 0$. 
The matching loss $L_{F_t}(\ah \mid \y) = \Delta_{F_t}\big(\y, \fs^*_t(\ah)\big)$
for the link function $\fs^*_t$ of $\Fs^*_t$ is convex in the activations $\ah$. 
However, by letting the temperature $t_2$ of $\fs^*_{t_2}$ 
be larger than the temperature $t_1$ of $F_{t_1}$, we
construct bounded non-convex losses with heavy-tailed transfer functions.

\vspace{-0.2cm}
\section{Tempered Matching Loss}
\vspace{-0.1cm}
\label{sec:temp}
We start by introducing a generalization of the relative entropy, denoted by $\Delta_{F_t}$, induced by a strictly convex function $F_t:\R_+^k \rightarrow \R$ with a temperature parameter $t \geq 0$.
The convex function $F_t$ is chosen so that its gradient takes the form\footnote{Here, the $\log_t$ function is applied elementwise.} $f_t(\y)\coloneqq \nabla F_t(\y) = \log_t \y$. Via simple integration, we obtain that
\[
F_t(\y) = \sum_{i=1}^k \big(y_i \log_t y_i + \tfrac{1}{2-t}(1 - y_i^{2-t})\big)\, .
\]
Indeed, $F_t$ is a convex function since $\nabla^2 F_t(\y) = \text{diag}(\y^{-t}) \succeq 0$ for any $\y \in \R^k_+$. 
In fact, $F_t$ is strongly convex, for $0 \leq t \leq 1$:

\begin{lemma}
\label{lem:strong}
The function $F_t$, with $0 \leq t \leq 1$, is
    $B^{-t}$--strongly convex over the set $\{\y \in
    \R_+^k:\, \Vert\y\Vert_{2-t} \leq B\}$ w.r.t. the $\LL_{2-t}$-norm.
\end{lemma}

See Appendix~\ref{app:strong} for a proof.
The Bregman divergence induced by $F_t$ is then given by
\begin{equation}
\begin{aligned}
\label{eq:div}
    \Delta_{F_t}(\y, \yh) &= \sum_{i=1}^k \big(y_i \log_t y_i - y_i \log_t \hat{y}_i - \tfrac{1}{2-t}y_i^{2-t} + \tfrac{1}{2-t}\hat{y}_i^{2-t}\big) \\ 
    &= \sum_{i=1}^k \Big(\tfrac{1}{(1-t)(2-t)}\,y_i^{2-t} -
    \tfrac{1}{1-t}\,y_i \hat{y}_i^{1-t} + \tfrac{1}{2-t}\,\hat{y}_i^{2-t}\Big).
\end{aligned}
\end{equation}
The second form may be recognized as $\beta$-divergence
\cite{beta}
with parameter $\beta = 2-t$. The divergence~\eqref{eq:div}
includes many well-known divergences such as squared
Euclidean, KL, and Itakura-Saito divergence as special
cases. A list of additional special cases is given in
Table~\ref{tab:special} of Appendix \ref{a:other}. 

The following corollary is the direct consequence of the
strong convexity of $F_t$, for $0 \leq t < 1$.

\begin{corollary}
\label{cor:bound}
Let $\max(\Vert\y\Vert_{2-t},\, \Vert\yh\Vert_{2-t}) \leq B$\, for\, $0 \leq t < 1$. Then
\[
\frac{1}{2B^t} \Vert \y - \yh\Vert^2_{2-t} \leq \Delta_{F_t}(\y, \yh) \leq \frac{B^t}{2\,(1-t)^2} \Vert \y^{1-t} - \yh^{1-t}\Vert^2_{\frac{2-t}{1-t}}\, .
\]
\end{corollary}
See Appendix~\ref{app:strong} for a proof. Thus for $0 \leq t < 1$, $\Delta_{F_t}(\y, \yh)$ is upper-bounded by $\frac{2\,B^{2-t}}{(1-t)^2}$. Note that boundedness on the simplex also implies boundedness in the $L_{2-t}$-ball of radius $1$. Thus, Corollary~\ref{cor:bound} immediately implies the boundedness of the divergence $\Delta_{F_t}(\y, \yh)$ with $0 \leq t < 1$ over the simplex.
Alternate parameterizations of the family $\{F_t\}$ of convex functions and their corresponding Bregman divergences are discussed in Appendix~\ref{a:other}.

\vspace{-0.1cm}
\subsection{Tempered softmax function}

Now, let us consider the convex function $F_t(\y)$ 
when its domain is restricted to the probability simplex $S^k$. 
We denote the constrained dual of $F_t(\y)$ by $\Fs_t^*(\a)$,
\begin{equation}
\label{eq:dual-Fst}
\Fs_t^*(\a) = \sup_{\y' \in S^k}\, 
    \big(\y'\cdot\a - F_t(\y')\big) = \sup_{\y' \in \R^k_+}\, \inf_{\lambda_t\in\R\vphantom{^k}}\,
    \big(\y'\cdot\a - F_t(\y') + \lambda_t\, \big(1-\sum_{i=1}^k y'_i\big)\big)\, .
\end{equation}
Following our discussion in Section~\ref{sec:bregman} and using~\eqref{eq:links}, the transfer function induced by $\Fs_t^*$ is%
\footnote{%
Note that due to the simplex constraint, the link function $\y = \fs^*_t (\a) = \nabla \Fs_t^*(\a)= \exp_t\big(\a - \lambda_t(\a)\big)$ is different from 
$f_t^{-1}(\a)=f_t^*(\a)= \nabla F_t^*(\a) = \exp_t(\a)$, i.e., the gradient of the unconstrained dual.}
\begin{align}
    \y = \exp_t\big(\a - \lambda_t(\a)\,\bm{1}\big),  \label{eq:prob}
    \quad\text{with $\lambda_t(\a)$~~s.t.}
    \quad
    \sum_{i=1}^k \exp_t\big(a_i - \lambda_t(\a)\big) = 1.
\end{align}

\vspace{-0.5cm}
\subsection{Matching loss of tempered softmax}

Finally, we derive the matching loss function $L_{F_t}$.
%
Plugging in~\eqref{eq:prob} into~\eqref{eq:div}, we have
$$
L_t(\ah \mid \y)
= \Delta_{F_t}\big(\y,\exp_t(\ah-\lambda_t(\ah))\big)
.
$$
Recall that by Proposition~\ref{lem:convexity}, this loss is convex in activations $\ah \, \in \text{\textbf{range}}((\fs_t^*)^{-1})$.
In general, $\lambda_t(\a)$ does not have a closed form solution.
However, it can be easily approximated via an iterative method, e.g., a binary search.
An alternative (fixed-point) algorithm for computing $\lambda_t(\a)$ for $t>1$ is given in 
Algorithm~\ref{alg:iterative} of Appendix \ref{a:alg}.

\vspace{-0.2cm}
\section{Robust Bi-Tempered Logistic Loss}

A more interesting class of loss functions can be obtained
by introducing a ``mismatch'' between the temperature of the divergence function~\eqref{eq:div} and the temperature of the probability assignment function, i.e. the tempered softmax~\eqref{eq:prob}. 
That is, we consider loss functions of the following type: 
\begin{equation}
    \label{eq:2temp}
\forall\, 0\! \leq\! t_1\! < 1\! <\! t_2\!:\,
    \loss(\ah \mid \y) := \Delta_{F_{\scalebox{0.6}{$t$}_{\scalebox{0.4}{$1$}}}} 
    \big(\y, \exp_{t_2}(\ah - \lambda_{t_2}(\ah))\big), \ 
	  \text{ with } \lambda_t(\ah) \ \
	  \text{ s.t. }
          \sum_{i=1}^k \exp_t\big(a_i - \lambda_t(\a)\big) = 1.
\end{equation}

We call this the \emph{Bi-Tempered Logistic Loss}. Note that
for the prescribed range of the two temperatures, 
the loss is bounded and has a heavier-tailed probability assignment
function compared to the vanilla softmax function. As
illustrated in our 2-dimensional example in Section~\ref{sec:intro}, 
both properties are crucial for handling noisy examples. 
The derivative of the bi-tempered loss is given in
Appendix \ref{a:deriv}.
In the following, we discuss the properties of this loss for classification.

\vspace{-0.2cm}
\subsection{Properness and Monte-Carlo sampling}

Let $\Puk(\x, y)$ denote the (unknown) joint probability
distribution of the observed variable $\x \in \R^m$ and
the class label $y \in [k]$. The goal of discriminative
learning is to approximate the posterior distribution of
the labels $\Puk(y \mid \x)$ via a parametric model
$P(y \mid \x;\Theta)$ parameterized by $\Theta$. 
Thus the model fitting can be expressed as minimizing the following expected loss between the data and the model's label probabilities
\begin{equation}
\label{eq:exp-loss}
  \EE_{\Puksub(\x)}\Big[\Delta \big(\Puk(y \mid \x), P(y \mid \x;\Theta)\big)\Big]\, ,  
\end{equation}
where $\Delta\big(\Puk(y \mid \x), P(y \mid \x;\Theta)\big)$ is any proper divergence 
measure between $\Puk(y \mid \x)$ and $P(y \mid \x;\Theta)$. 
We use $\Delta:=\Delta_{F_{\scalebox{0.6}{$t$}_{\scalebox{0.4}{$1$}}}}$ 
as the divergence and
$P(y = i \mid \x;\Theta) \coloneqq P(i \mid \x;\Theta) 
= \exp_{t_2}(\hat{a}_i - \lambda_{t_2}(\ah)),$ 
where $\ah$ is the activation vector
of the last layer given input $\x$ and $\Theta$ is the set
of all weights of the network. 
Ignoring the constant terms w.r.t. $\Theta$, our loss~\eqref{eq:exp-loss} becomes
\begin{subequations}
\begin{align}
    & \EE_{\Puksub(\x)}\Big[\sum_i \big(-\Puk(i \mid \x) \log_t P(i \mid \x;\Theta) + \frac{1}{2-t}\, P(i \mid \x;\Theta)^{2-t}\big) \Big]\label{eq:exp-robust1}\\
    & \quad\quad = -\EE_{\Puksub(\x, y)}\Big[\log_t P(y \mid \x;\Theta)\Big] + \EE_{\Puksub(\x)}\Big[\frac{1}{2-t}\, \sum_i P(i \mid \x;\Theta)^{2-t}\big)\Big]\label{eq:exp-robust2}\\
    &\quad\quad \approx \frac{1}{N}\sum_n \big(-\log_t P(y_n \mid \x_n;\Theta) + \frac{1}{2-t}\, \sum_i P(i \mid \x_n;\Theta)^{2-t}\big)\label{eq:exp-robust3}
\end{align}
\end{subequations}
where from~\eqref{eq:exp-robust2} to~\eqref{eq:exp-robust3}, we perform a Monte-Carlo
approximation of the expectation w.r.t. $\Puksub(\x, y)$ using
samples $\{(\x_n, y_n)\}_{n=1}^N$. 
Thus,~\eqref{eq:exp-robust3} is an unbiased approximate of the expected loss~\eqref{eq:exp-loss}, thus is a \emph{proper} loss~\cite{proper}. 

Following the same approximation steps for the Tsallis divergence, we have
\begin{align*}
    \EE_{\Puksub(\x)}\Big[\underbrace{- \sum_i \Puk(i \mid \x) \log_t\frac{P(i \mid \x;\Theta)}{\Puk(i \mid \x)}}_{\Delta_t^{\text{\tiny Tsallis}}\big(\Puk(y \mid \x),\, P(y \mid \x;\Theta)\big)}\Big] 
    & \approx -\frac{1}{N}\, \sum_n \log_t\frac{P(y_n \mid \x_n;\Theta)}{\Puk(y_n \mid \x_n)}\, ,
\end{align*}
which, due to the fact that $\log_t\frac{a}{b}\neq \log_t a - \log_t b$ in general, requires access to the (unknown) conditional distribution $\Puk(y \mid \x)$. Thus, the approximation $-\frac{1}{N}\, \sum_n \log_t P(y_n \mid \x_n;\Theta)$ proposed in~\cite{ourtsallis} by approximating $\Puk(y_n \mid \x_n)$ by $1$ is not an unbiased estimator of~\eqref{eq:exp-loss} and therefore, not proper.

\subsection{Bayes-risk consistency}

Another important property of a multiclass loss is the Bayes-risk consistency~\cite{bayes-multi}. Bayes-risk consistency of the two-temperature logistic loss based on the Tsallis divergence was shown in~\cite{ourtsallis}. As expected, the tempered Bregman loss~\eqref{eq:2temp} is also Bayes-risk consistent, even in the non-convex case.
\begin{proposition}
\label{prop:bayes}
The multiclass bi-tempered logistic loss $\loss(\ah\,\vert\, y)$ is Bayes-risk consistent.
\end{proposition}

\vspace{-0.2cm}
\section{Experiments}
\vspace{-0.1cm}
We demonstrate the practical utility of the bi-tempered logistic loss function
on a wide variety of image classification tasks. For moderate size experiments, we use
MNIST dataset of handwritten digits~\cite{mnist} and CIFAR-100, which contains real-world images from 100 different classes
\cite{cifar100}. We use  ImageNet-2012~\cite{imagenet} for large scale image
classification, having 1000 classes. All experiments are carried out using the TensorFlow~\cite{tensorflow} framework. We use P100 GPU's for small scale experiments and Cloud TPU-v2 for larger scale ImageNet-2012 experiments. An implementation of the bi-tempered logistic loss is available online at: \url{https://github.com/google/bi-tempered-loss}.

\vspace{-0.2cm}

\begin{figure*}[t]
\vspace{-1cm}
\hspace{-2cm}
\footnotesize
\begin{tabular}{m{0.05\textwidth} m{0.32\textwidth} m{0.32\textwidth} m{0.32\textwidth}}
  & \multicolumn{1}{c}{\quad\quad Training Set (Noise-free)}
  & \multicolumn{1}{c}{\quad\quad Training Set (Noisy)}
    &  \multicolumn{1}{c}{\quad\quad Test set (Noise-free)}\\
    \multicolumn{1}{c}{\rotatebox[origin=c]{90}{Logistic}}\hspace{-0.8cm} &
\includegraphics[height=0.3\textwidth]{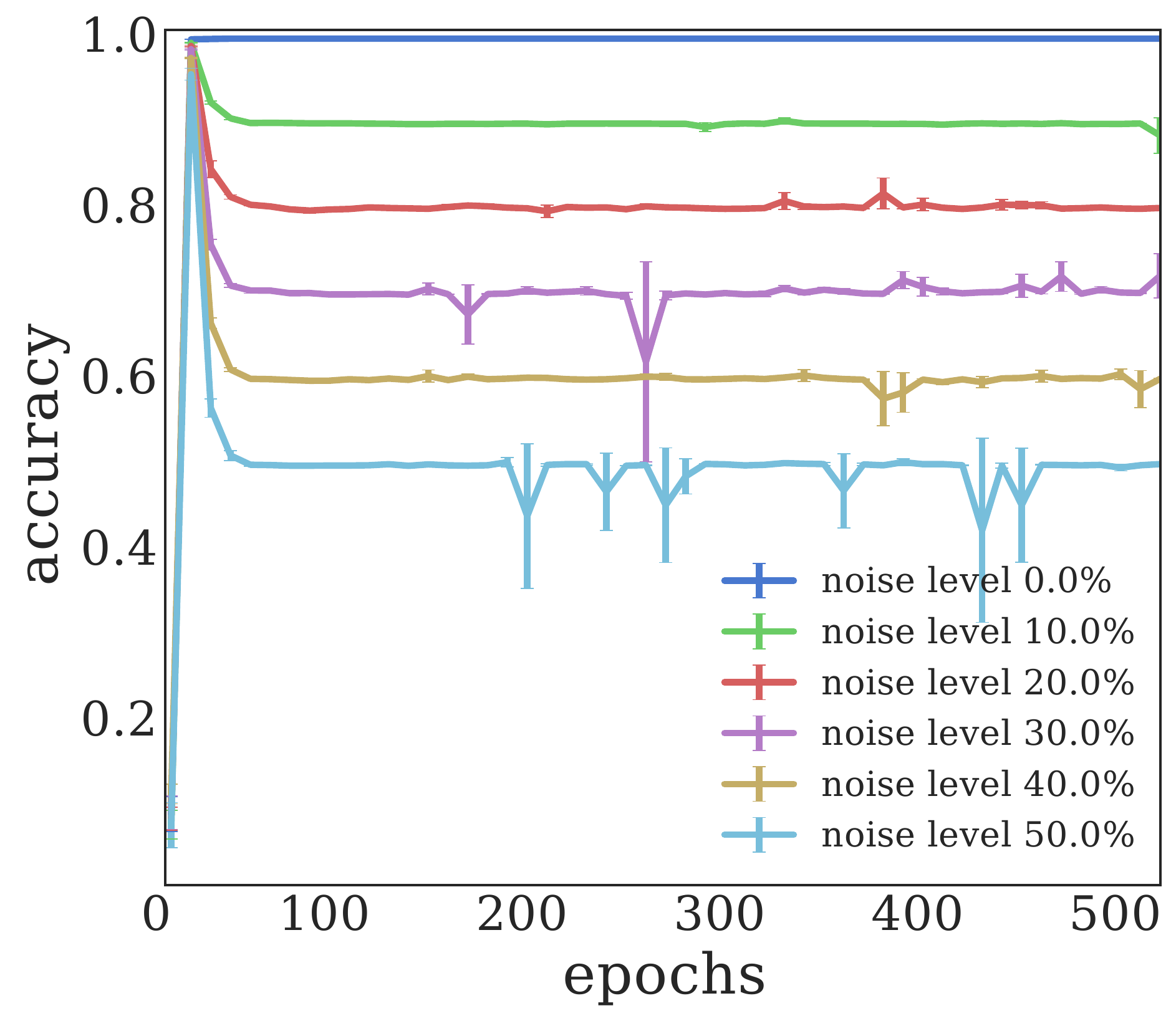} &
\includegraphics[height=0.3\textwidth]{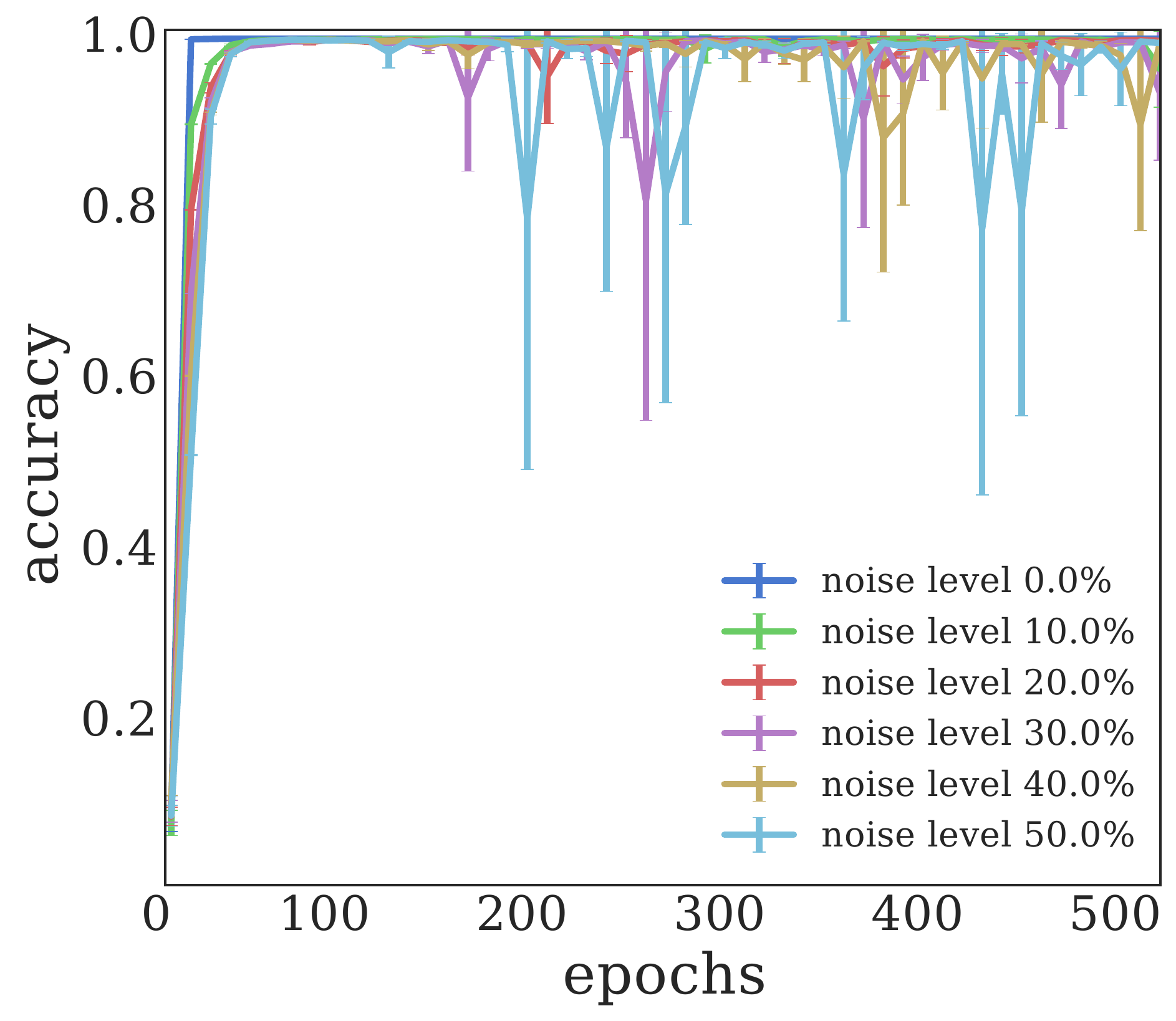} &
\includegraphics[height=0.3\textwidth]{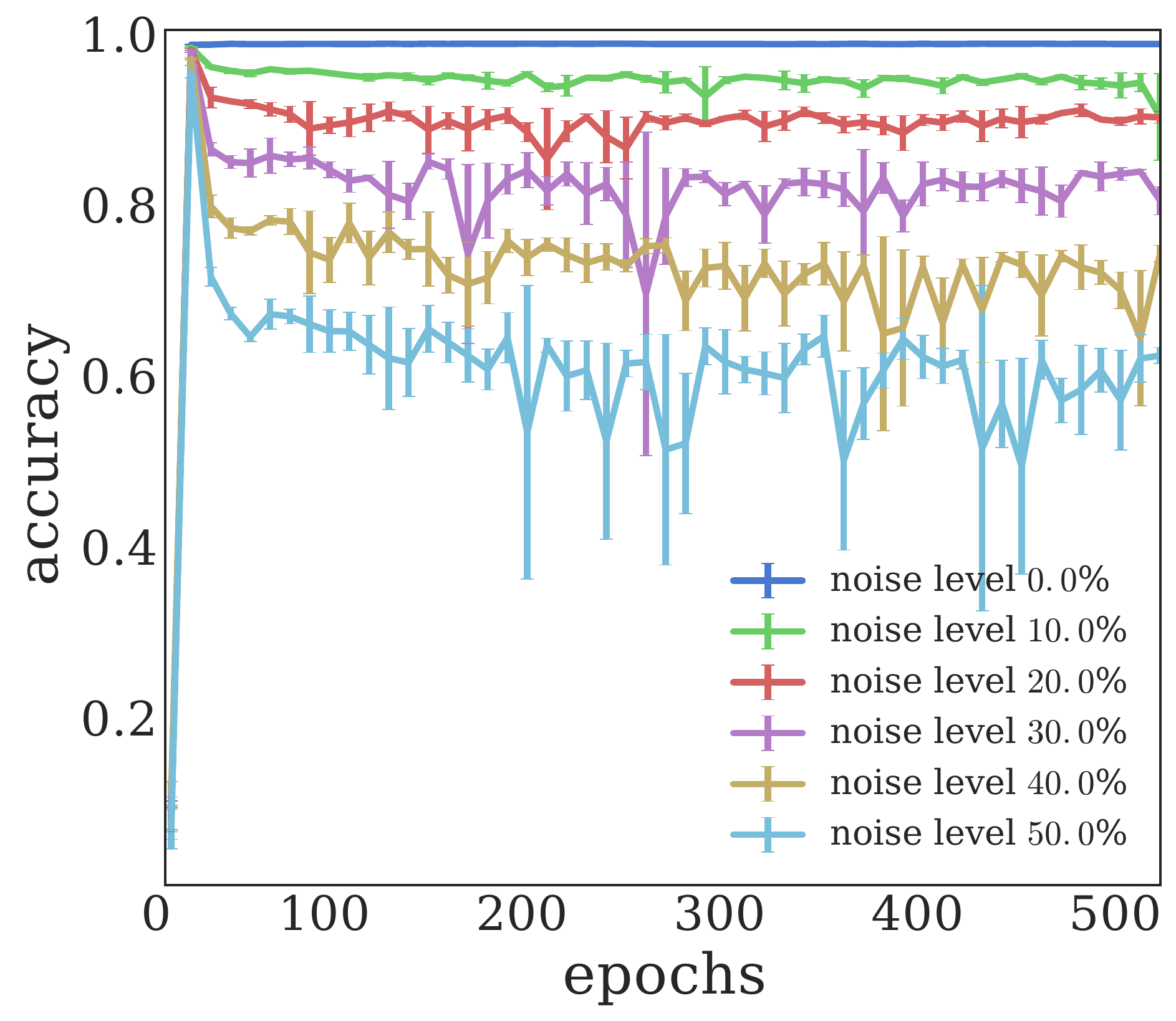}\\
\multicolumn{1}{c}{\rotatebox[origin=c]{90}{Bi-Tempered}}\hspace{-0.8cm} &
\includegraphics[height=0.3\textwidth]{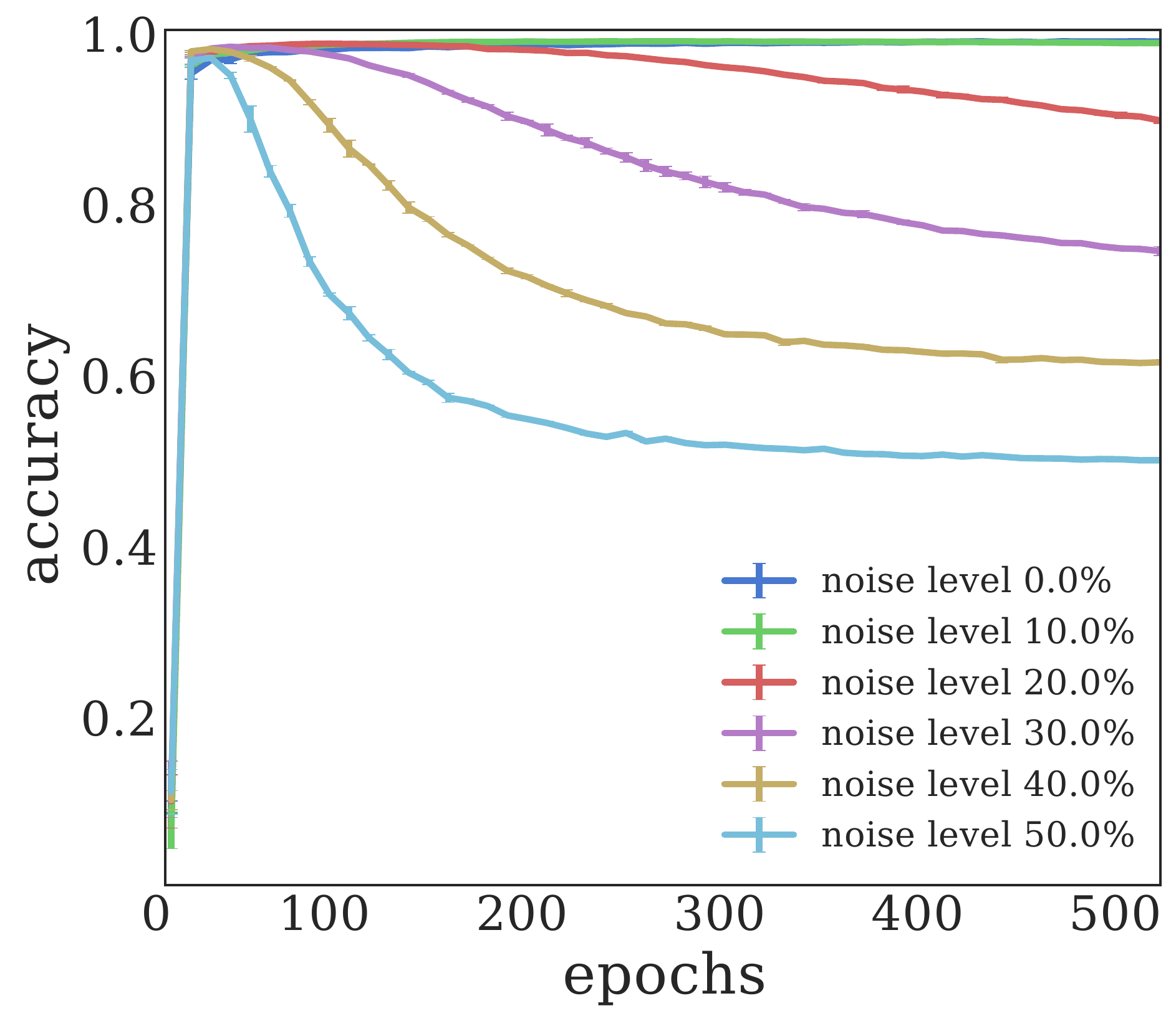} &
\includegraphics[height=0.3\textwidth]{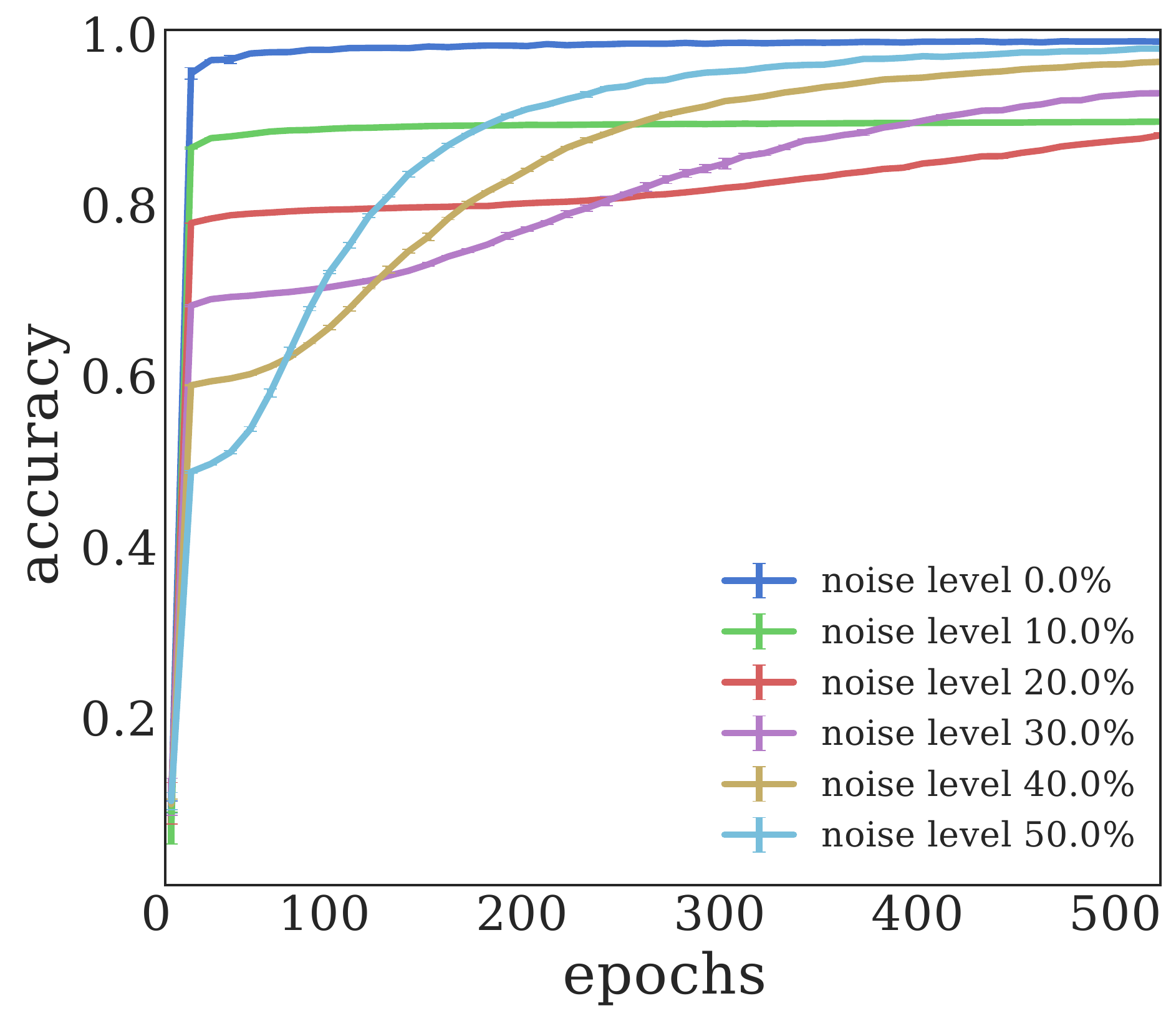} &
\includegraphics[height=0.3\textwidth]{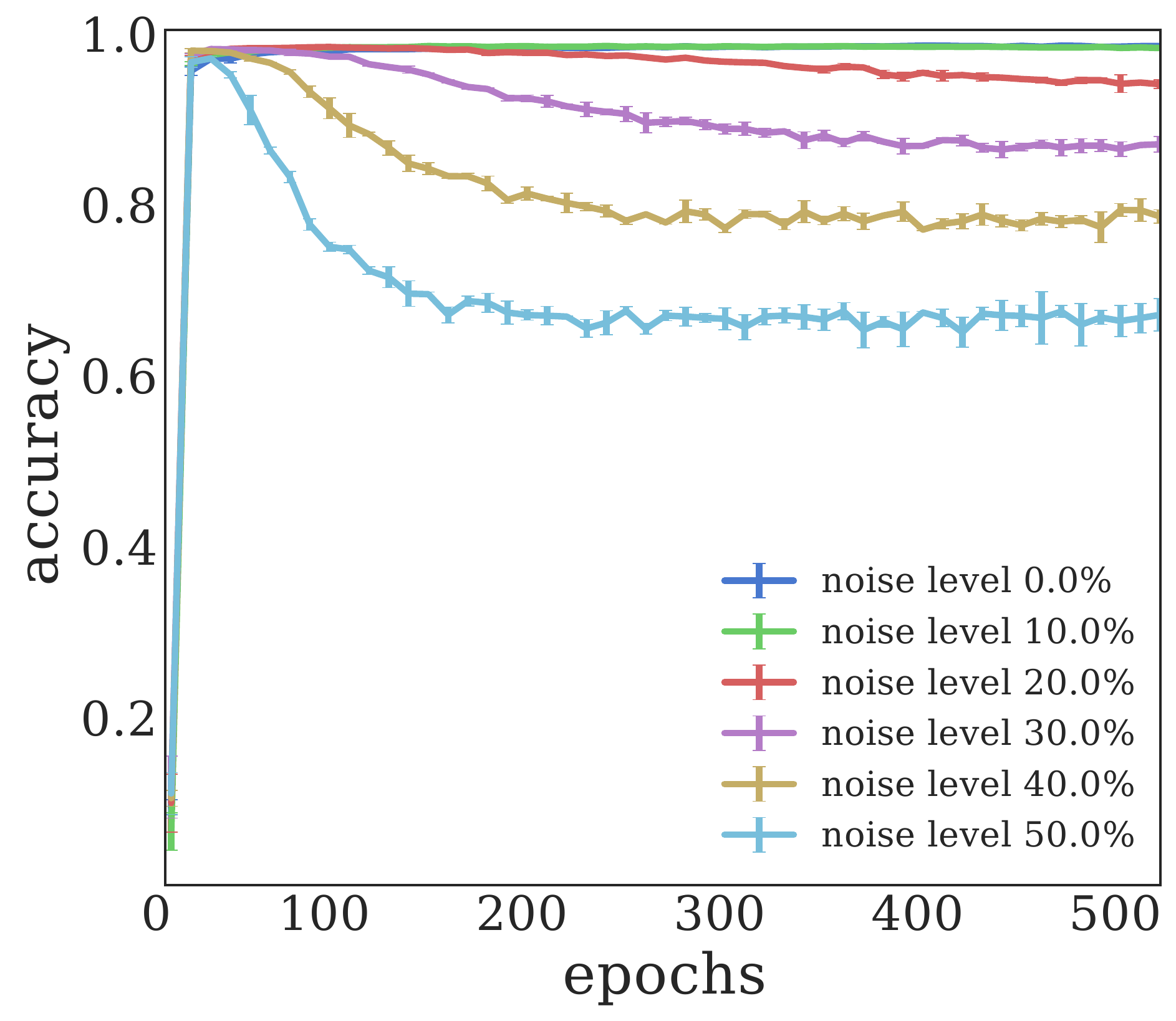}\\
& \multicolumn{1}{c}{{\quad\quad\quad (a)}} & \multicolumn{1}{c}{{\quad\quad \quad(b)}} &  \multicolumn{1}{c}{{\quad\quad\quad (c)}}\\
\end{tabular}
\caption{Top-1 accuracy of the models trained using the logistic loss (top) and the bi-tempered loss with $(t_1, t_2) = (0.5, 4.0)$ (bottom) on the noisy MNIST dataset: accuracy on (a) noise-free training set, (b)  noisy training set, (c) and noise-free test set. Initially, both models provide better generalization but gradually overfit to the label noise. However, the overfitting for the logistic loss happens much earlier during the optimization. The variance of the model is also much higher for the logistic loss. The bi-tempered loss provides better generalization accuracy overall.
    }\label{fig:accuracy}
    \vspace{-1mm}
\end{figure*}
\subsection{Corrupted labels experiments}
For our moderate size datasets, i.e. MNIST and CIFAR-100, we introduce noise by artificially corrupting a fraction of the labels and producing a new set of labels for each noise level. For all experiments, we compare our bi-tempered loss function against the logistic loss. For MNIST, we use a CNN with two convolutional layers of size $32$ and $64$ with a mask size of $5$, followed by two fully-connected layers of size $1024$ and $10$. We apply max-pooling after each convolutional layer with a window size equal to $2$ and use dropout during training with keep probability equal to $0.75$. We use the AdaDelta optimizer~\cite{adadelta} with $500$ epochs and batch size of $128$ for training.
For CIFAR-100, we used a Resnet-56 \cite{resnet} model without batch norm from \cite{HardtM16} with SGD + momentum optimizer trained for 50k steps with batch size of 128  and use the standard learning rate stair case decay schedule. For both experiments, we report the test accuracy of the checkpoint which yields the highest accuracy on an identically label-noise corrupted validation set. We search over a set of learning rates for each experiment. For both experiments, we exhaustively search over 
a number of temperatures within the range $[0.5, 1)$ and $(1.0, 4.0]$ for $t_1$ and $t_2$, respectively. The results are presented in Table \ref{tab:synthetic_results} where we report the top-1 accuracy on a clean test set. As can be seen, the bi-tempered loss outperforms the logistic loss for all noise levels (including the noise-free case for CIFAR-100). Using our bi-tempered loss function the model is able to continue to perform well even for high levels of label noise whereas the accuracy of the logistic loss drops immediately with a much smaller level of noise. Additionally, in Figure~\ref{fig:accuracy}, we illustrate the top-1 accuracy on the noise-free and noisy training set, as well the accuracy on the (noise-free) test set for both losses as a function of number of epochs. As can be seen from the figure, initially both models yield a relatively higher test accuracy, but gradually overfit to the label noise in the training set over time. The overfitting to the noise deteriorates  the generalization capacity of the models. However, overfitting to the noise happens earlier in the training and is much severe in case of the logistic loss. As a result, the final test accuracy (after 500 epochs) is comparatively much lower than the bi-tempered loss as the noise level increases. Finally, note that the variance of the model is also considerably higher for the logistic loss. This confirms that the bi-tempered loss results in more stable models when the data is noise-corrupted.

\begin{table}[h]
{\small
\begin{center}
\begin{tabular}{llcccccc}
\toprule
\multirow{2}{*}{Dataset} & \multirow{2}{*}{Loss} & \multicolumn{6}{c}{Label Noise Level}\\
\cmidrule(lr){3-8}
& & 0.0 & 0.1 & 0.2 & 0.3 & 0.4 & 0.5\\
\midrule
\midrule
\multirow{2}{*}{MNIST} &
\multirow{1}{*}{Logistic} & \textbf{99.40} & 98.96 & 98.70 & 98.50 & 97.64 & 96.13\\
                    \cmidrule(lr){3-8}
                    & \multirow{1}{*}{Bi-Tempered ($0.5, 4.0$)}   & 99.24 & \textbf{99.13} &  \textbf{99.02} &  \textbf{98.62} & \textbf{98.56} & \textbf{97.69} \\
                                                            
\midrule
\multirow{2}{*}{CIFAR-100}  & \multirow{1}{*}{Logistic} & 74.03 & 69.94 & 66.39 & 63.00 & 53.17 & 52.96\\
                    \cmidrule(lr){3-8}
                       & \multirow{1}{*}{Bi-Tempered ($0.8$, $1.2$)}  & \textbf{75.30} & \textbf{73.30} &  \textbf{70.69} &  \textbf{67.45} & \textbf{62.55}  & \textbf{57.80}\\
                                                    
\bottomrule
\end{tabular}
\end{center}
}
\caption{Top-1 accuracy on a clean test set for MNIST and CIFAR-100 datasets where a fraction of the training labels are corrupted.}
\label{tab:synthetic_results}
\end{table}

\vspace{-0.5cm}
\subsection{Large scale experiments}

We train state-of-the-art Resnet-18 and Resnet-50 models on the ImageNet-2012 dataset. Note that the ImageNet-2012 dataset is inherently noisy due to some amount of mislabeling. We train on a 4x4 CloudTPU-v2 device with a batch size of 4096. All experiments were trained for 180 epochs, and use the SGD + momentum optimizer with staircase learning rate decay schedule. The results are presented in Table~\ref{tab:imagenet_results}. 
For both architectures we see a significant gain of the
robust bi-tempered loss method in top-1 accuracy. 

\begin{table}[h]
{\small
\begin{center}
\begin{tabular}{ccc}
\toprule
Model & Logistic  & Bi-tempered (0.9,1.05)\\
\midrule
\multirow{1}{*}{Resnet18}  & $71.333 \pm 0.069$  &  $\bm{71.618} \pm 0.163$\\
\midrule
\multirow{1}{*}{Resnet50} &  $76.332 \pm 0.105$ &  $\bm{76.748} \pm 0.164$ \\
\bottomrule
\end{tabular}
\end{center}
}
\caption{Top-1 accuracy on ImageNet-2012 with Resnet-18 and 50 architectures.}
\label{tab:imagenet_results}
\end{table}




\section{Conclusion and Future Work}
Neural networks on large standard datasets have been
optimized along with a large variety of variables such as
architecture, transfer function, choice of optimizer, and label smoothing to name just a few.
We proposed a new variant by training the network with
tunable loss functions.
We do this by first developing convex loss functions based on
temperature dependent logarithm and exponential functions. 
When both temperatures are the same, then a construction
based on the notion of ``matching loss'' leads to loss
functions that are convex in the last layer.
However by letting the temperature of the new tempered softmax function
be larger than the temperature of the tempered log function used in
the  divergence, we construct tunable losses that are non-convex in the last layer.
Our construction remedies two issues simultaneously: 
we construct bounded tempered loss functions that can handle large-margin outliers and introduce heavy-tailedness in our
new tempered softmax function that seems to handle small-margin mislabeled examples.
At this point, we simply took a number of benchmark datasets and networks for these datasets that have been heavily
optimized for the logistic loss paired with vanilla softmax
and simply replaced the loss in the last layer by our new
construction. By simply trying a number of temperature
pairs, we already achieved significant improvements.
We believe that with a systematic ``joint optimization'' of all
commonly tried variables, significant further improvements
can be achieved. This is of course a more long-term goal.
We also plan to explore the idea of annealing the temperature parameters over the training process.

\section*{Acknowledgement}
We would like to thank Jerome Rony for pointing out the errors in the MNIST experiments.

%


\bibliographystyle{plain}
\bibliography{refs}

\begin{thebibliography}{10}

\bibitem{tensorflow}
Mart\'{\i}n Abadi, Ashish Agarwal, Paul Barham, Eugene Brevdo, Zhifeng Chen,
  Craig Citro, Greg~S. Corrado, Andy Davis, Jeffrey Dean, Matthieu Devin,
  Sanjay Ghemawat, Ian Goodfellow, Andrew Harp, Geoffrey Irving, Michael Isard,
  Yangqing Jia, Rafal Jozefowicz, Lukasz Kaiser, Manjunath Kudlur, Josh
  Levenberg, Dan Man\'{e}, Rajat Monga, Sherry Moore, Derek Murray, Chris Olah,
  Mike Schuster, Jonathon Shlens, Benoit Steiner, Ilya Sutskever, Kunal Talwar,
  Paul Tucker, Vincent Vanhoucke, Vijay Vasudevan, Fernanda Vi\'{e}gas, Oriol
  Vinyals, Pete Warden, Martin Wattenberg, Martin Wicke, Yuan Yu, and Xiaoqiang
  Zheng.
\newblock {TensorFlow}: Large-scale machine learning on heterogeneous systems,
  2015.
\newblock Software available from tensorflow.org.

\bibitem{ourtsallis}
Ehsan Amid, Manfred~K. Warmuth, and Sriram Srinivasan.
\newblock Two-temperature logistic regression based on the {T}sallis
  divergence.
\newblock In {\em 22nd International Conference on Artificial Intelligence and
  Statistics (AISTATS 19)}, 2019.

\bibitem{buja}
Andreas Buja, Werner Stuetzle, and Yi~Shen.
\newblock Loss functions for binary class probability estimation and
  classification: Structure and applications.
\newblock Technical report, University of Pennsylvania, November 2005.

\bibitem{beta}
Andrzej Cichocki and Shun-ichi Amari.
\newblock Families of alpha-beta-and gamma-divergences: Flexible and robust
  measures of similarities.
\newblock {\em Entropy}, 12(6):1532--1568, 2010.

\bibitem{imagenet}
J.~Deng, W.~Dong, R.~Socher, L.-J. Li, K.~Li, and L.~Fei-Fei.
\newblock {ImageNet: A Large-Scale Hierarchical Image Database}.
\newblock In {\em CVPR09}, 2009.

\bibitem{ding}
Nan Ding.
\newblock {\em Statistical machine learning in the t-exponential family of
  distributions}.
\newblock PhD thesis, Purdue University, 2013.

\bibitem{tlogistic}
Nan Ding and S.~V.~N. Vishwanathan.
\newblock $t$-logistic regression.
\newblock In {\em Proceedings of the 23th International Conference on Neural
  Information Processing Systems}, NIPS'10, pages 514--522, Cambridge, MA, USA,
  2010.

\bibitem{HardtM16}
Moritz Hardt and Tengyu Ma.
\newblock Identity matters in deep learning.
\newblock {\em International Conference on Learning Representations (ICLR)},
  2017.

\bibitem{resnet}
Kaiming He, Xiangyu Zhang, Shaoqing Ren, and Jian Sun.
\newblock Deep residual learning for image recognition.
\newblock In {\em Proceedings of the IEEE conference on computer vision and
  pattern recognition}, pages 770--778, 2016.

\bibitem{match}
D.~P. Helmbold, J.~Kivinen, and M.~K. Warmuth.
\newblock Relative loss bounds for single neurons.
\newblock {\em IEEE Transactions on Neural Networks}, 10(6):1291--1304,
  November 1999.

\bibitem{matchmulti}
J.~Kivinen and M.~K. Warmuth.
\newblock Relative loss bounds for multidimensional regression problems.
\newblock {\em Machine Learning}, 45(3):301--329, 2001.

\bibitem{cifar100}
Alex Krizhevsky.
\newblock Learning multiple layers of features from tiny images.
\newblock Technical report, Citeseer, 2009.

\bibitem{mnist}
Yann LeCun and Corinna Cortes.
\newblock {The MNIST database of handwritten digits}, 1999.

\bibitem{long}
Philip~M Long and Rocco~A Servedio.
\newblock Random classification noise defeats all convex potential boosters.
\newblock In {\em Proceedings of the 25th international conference on Machine
  learning}, pages 608--615. ACM, 2008.

\bibitem{texp1}
Jan Naudts.
\newblock Deformed exponentials and logarithms in generalized thermostatistics.
\newblock {\em Physica A}, 316:323--334, 2002.

\bibitem{reidwill}
M.~D. Reid and R.~C. Williamson.
\newblock Surrogate regret bounds for proper losses.
\newblock In {\em Proceedings of the 26th International Conference on Machine
  Learning (ICML'09)}, pages 897--904, 2009.

\bibitem{olsurvey}
Shai Shalev-Shwartz et~al.
\newblock Online learning and online convex optimization.
\newblock {\em Foundations and Trends{\textregistered} in Machine Learning},
  4(2):107--194, 2012.

\bibitem{bayes-multi}
Ambuj Tewari and Peter~L Bartlett.
\newblock On the consistency of multiclass classification methods.
\newblock {\em Journal of Machine Learning Research}, 8(May):1007--1025, 2007.

\bibitem{proper}
Robert~C. Williamson, Elodie Vernet, and Mark~D. Reid.
\newblock Composite multiclass losses.
\newblock {\em Journal of Machine Learning Research}, 17(223):1--52, 2016.

\bibitem{adadelta}
Matthew~D Zeiler.
\newblock Adadelta: an adaptive learning rate method.
\newblock {\em arXiv preprint arXiv:1212.5701}, 2012.

\bibitem{zhang2018}
Zhilu Zhang and Mert Sabuncu.
\newblock Generalized cross entropy loss for training deep neural networks with
  noisy labels.
\newblock In {\em Advances in Neural Information Processing Systems}, pages
  8778--8788, 2018.

\end{thebibliography}

\newpage
\appendix

\section{An Iterative Algorithm for Computing the Normalization}
\label{a:alg}
\begin{algorithm}[h]
    \caption{Iterative algorithm for computing
    $\lambda_t(\a)$ (from~\cite{ourtsallis})}
\label{alg:iterative}
\begin{algorithmic}
\STATE {\bfseries Input:} Vector of activations $\a$, temperature $t > 1$
\STATE $\mu \gets \max(\a)$
\STATE $\tilde{\a} \gets \a-\mu$
\WHILE{$\tilde{\a}$ not converged}
\STATE $Z(\tilde{\a})\gets \sum_{i=1}^k \exp_{t}(\tilde{a}_i)$
\STATE$\tilde{\a}\gets Z(\tilde{\a})^{1-t} (\a-\mu\,\one)$
\ENDWHILE
    \STATE {\bfseries Return:} $\lambda_t(\a)\gets -\log_t\frac{1}{Z(\tilde{\a})}+\mu$
\end{algorithmic}
\end{algorithm}

\section{Strong Convexity and Smoothness}
\label{app:strong}
The following material for strong convexity and strong smoothness are adopted from~\cite{olsurvey}.
\begin{definition}[$\sigma$-Strong Convexity] 
    A continuous function $F$ is $\sigma$-strongly convex w.r.t.~the norm $\Vert\cdot\Vert$ over the convex set $\Set$ if $\Set$ is contained in the domain of $F$ and for any $\ub, \, \vb\in \Set$, we have
    \[
        F(\vb) \geq F(\ub) + \nabla F(\ub)\cdot(\vb - \ub) + \frac{\sigma}{2}\,\Vert\vb - \ub\Vert^2\, .
    \]
\end{definition}
\begin{lemma}
    Assume $F$ is twice differentiable. Then $F$ is $\sigma$-strongly convex if
    \[
        \big(\nabla^2 F(\ub)\,\vb\big)\cdot\vb \geq \sigma\,\Vert \vb \Vert^2
,\,\,\,\forall \ub,\, \vb \in \Set\,  .
    \]
\end{lemma}
\begin{lemma}
    \label{lem:suff}
Let $F$ be a $\sigma$-strongly convex function over the  non-empty convex set $\Set$. For all $\ub,\, \vb\in \Set$, we have
    \[
        \frac{\sigma}{2}\, \Vert\ub - \vb\Vert^2\,\leq\, \Delta_F(\vb, \ub)\,  .
    \]
\end{lemma}

\begin{proof}[\textbf{Proof of Lemma~\ref{lem:strong}.}]
We have $\nabla^2\, F(\ub) = \text{diag}(\ub^{-t})$. Applying Lemma~\ref{lem:suff}, note that the function 
$$
(\nabla^2 F_t(\ub)\cdot\vb)\cdot \vb = \sum_i \frac{v_i^2}{u_i^{t}}\, ,
$$
is unbounded over the set $\Set = \{\vb \in \R_+^d:\, \Vert \vb \Vert_{2-t} \leq B\}$ and the minimum happens at the boundary $\{\Vert \vb \Vert_{2-t} = B\}$.
$$
\min_{\vb} \sum_i \frac{v_i^2}{u_i^{t}} + \gamma\, (\sum_i v_i^{2-t} - 1)\, \Rightarrow \,\vb = B\,\frac{\ub}{\Vert \ub \Vert_{2-t}}\, ,
$$
where $\gamma$ is the Lagrange multiplier. Plugging in the solution yields $\sum_i \frac{v_i^2}{u_i^{t}} \geq \frac{1}{B^t}\, \Vert \vb \Vert^2_{2-t}$.
\end{proof}

\begin{definition}[$\sigma$-Strong Smoothness]
A function differentiable function $G$ is $\sigma$-strongly smooth w.r.t. the norm $\Vert\cdot\Vert$ if 
    \[
        \Delta_G(\vb, \ub) \leq \frac{\sigma}{2}\, \Vert\vb - \ub\Vert^2\, .
    \]
\end{definition}

\begin{lemma}
    Let $F$ be a closed and convex function. Then $F$ is $\sigma$-strongly convex w.r.t. the $\vert\cdot\Vert$ if and only if $F^*$, the dual of $F$, is $\frac{1}{\sigma}$-strongly smooth w.r.t. the dual norm $\Vert\cdot\Vert_*$.
\end{lemma}

\begin{proof}[\textbf{Proof of Corollary~\ref{cor:bound}.}]
Note that using the duality of the Bregman divergences, we have
\[
\Delta_{F_t}(\y, \yh) = \Delta_{F^*_t}(f_t(\yh),\, f_t(\y)) = \Delta_{F^*_t}(\log_t(\yh), \log_t(\y)) \,.
\]
Using the strong convexity of $F_t$ and strong smoothness of $F_t^*$, we have
\[
\frac{1}{2B^t}\, \Vert\y-\yh\Vert_{2-t}^2 \leq \Delta_{F_t}(\y, \yh) \leq \frac{B^t}{2}\, \Vert\log_t\y-\log_t\yh\Vert_{\frac{2-t}{1-t}}^2\,.
\]
Note that $\Vert\cdot\Vert_{2-t}$ and $\Vert\cdot \Vert_{\frac{2-t}{1-t}}$ are dual norms. Substituting the definition of $\log_t$ to the right-hand-side, we have
\[
\frac{B^t}{2}\Vert\log_t\y-\log_t\yh\Vert_{2-t}^2 = \frac{B^t}{2\,(1-t)^2}\, \Vert\y^{1-t} - \yh^{1-t}\Vert_{\frac{2-t}{1-t}}^2 \leq \frac{B^t}{2\,(1-t)^2}\, \big(2\, B^{1-t}\big)^2 = \frac{2\, B^{2-t}}{(1-t)^2}\, .
\]
\end{proof}

\section{Other Tempered Convex Functions}
\label{a:other}
We begin with a list of interesting temperature choices for
the convex function $F_t$ and its induced divergence:

\begin{table}[h]
\begin{center}
    \begin{tabular}{|clll|}
\toprule
    $t$ & $F_t(\y)$ & $\Delta_{F_t}(\y,\yh)$ & Name\\
 \midrule
    0&$\frac{1}{2}\Vert\y\Vert_2^2$ &
    $\frac{1}{2}\Vert\y-\yh\Vert_2^2$ & Euclidean\\[1mm]
    
    $\frac{1}{2}$&$\frac{1}{3}\,\sum_i(4\,y_i^{\frac{4}{3}}
    - 6\,y_i + 2)$ & $\sum_i (\frac{4}{3} y_i^{\frac{3}{2}}
     - 2 y_i \sqrt{\hat{y}_i} + \frac{3}{2}
     \hat{y}_i^{\frac{3}{2}})$ & \\[1mm]
    
    1&$\sum_i (y_i \log y_i - y_i + 1)$ & $\sum_i (y_i \log
    \frac{y_i}{\hat{y}_i} - y_i + \hat{y}_i)$ & KL-divergence\\[1mm]
    
     $\frac{3}{2}$&$\sum_i(-4\,y_i^{\frac{3}{2}} + 2\,y_i +
     2)$ & $2\sum_i \frac{(\sqrt{y_i} -
     \sqrt{\hat{y}_i})^2}{\sqrt{\hat{y}_i}}$ & Squared Xi on roots\\[1mm]
    
    $2$&$\sum_i (-\log y_i + y_i)$ & $\sum_i
    (\frac{y_i}{\hat{y}_i} - \log\frac{y_i}{\hat{y}_i} -
    1)$ & Itakura-Saito\\[1mm]
    
     $3$&$\frac{1}{2}\,\sum_i(-\frac{1}{y_i} + y_i - 2)$ &
     $\frac{1}{2}\sum_i (\frac{1}{y_i} - \frac{2}{\hat{y}_i}
     + \frac{y_i}{\hat{y}_i^2})$ & Inverse\\[1mm]
\bottomrule
\end{tabular}
\end{center}
    \caption{Some special cases of the tempered Bregman divergence.}\label{tab:special}
\end{table}

In the construction of the convex function family $F_t$ 
we used $F_t(x)=\int \log_t(x)$ exploiting the fact
that $\log_t(x)$ is strictly increasing. 
We can also define an alternative convex function family $\Ft_t$
by utilizing the convexity (respectively, concavity) of the $\log_t$ 
function for values of $t \geq 0$ (respectively, $t \leq 0$):
$$
\Ft_t(\y) = -\frac{1}{t}\sum_i (\log_t y_i - y_i + 1) = -\frac{1}{t\,(1-t)}\sum_i (y_i^{1-t} - y_i)\, .
$$
Note that $\ft_t(\y) \coloneqq \nabla\Ft(\y) =  \frac{1 - \y^{-t}}{t}$ and $\nabla^2 \Ft_t(\y) = \text{diag}(\y^{-(1 + t)})$, thus $\Ft_t$ is indeed a strictly convex function. 
The following proposition shows that the Bregman divergence induced by 
the original and the alternate convex function are related by a temperature shift:
\begin{proposition}
For the Bregman divergence induced by the convex function $\Ft_t$, we have
\[\forall \y,\,\yh\in \R_+^k:\;\;
\Delta_{\Ft_t}(\y, \yh)\;\; =\;\; 
\frac{1}{t}\,\sum_i (\log_t \hat{y}_i - \log_t y_i + (y_i - \hat{y}_i)\,\hat{y}_i^{-t})
\;\;=\;\; \Delta_{F_{t+1}}(\y, \yh)\, .
\]
\end{proposition}
The $\Ft_t$ function is also related to the negative Tsallis entropy over the 
probability measures $\y \in \Delta_+^k$ defined as
\[
    - H_t^{\text{\tiny Tsallis}}(\y) =
      \frac{1}{1-t}\big(1-\sum_i y_i^{t}\big)
    = - \sum_i y_i \log_t\frac{1}{y_i}\, .
\]
Note that $(-H_t^{\text{\tiny Tsallis}} - (1-t)\,\Ft_{1-t})$ is an affine function. 
Thus, the Bregman Divergence induced by $\Ft_t$,
and the one induced by
$-H_t^{\text{\tiny Tsallis}}$ are also equivalent up to a
scaling and a temperature shift. 
Thus, both functions $F_t$ and $\Ft_t$ can be viewed as some
generalized negative entropy functions. 
Note that the Bregman
divergence induced by $-H_t^{\text{\tiny Tsallis}}$ is
different from the Tsallis divergence over the simplex, defined as
\[
\Delta_t^{\text{\tiny Tsallis}}(\y, \yh) = -\sum_i y_i \log_t\frac{\hat{y}_i}{y_i} = \sum_i y_i^t\, (\log_t y_i - \hat{y}_i)\, .
\]

\section{Convexity of the Tempered Matching Loss}
The convexity of the loss function $\Delta_{F_t}\big(\y,\exp_t(\ah-\lambda_t(\ah)\big)$ with $t \geq 1$ for $\ah \in \R^\nclass$ immediately follows from the definition of the matching loss. A more subtle case occurs when $0 \leq t < 1$. Note that the range of the combined function $\log_t \circ \exp_t$ does not cover all $\R^k$ as the $\log_t$ function is bounded from below by $-\frac{1}{1-t}$. Therefore, $\textbf{range}(\log_t \circ \exp_t) = \{\a' \in \R^\nclass\mid  -\frac{1}{1-t}\leq \a'\}$. 

\begin{remark}
The normalization function $\lambda_t(\a)$ satisfies: $\lambda_t(\a + b\, \one) = \lambda_t(\a) + b\,\,\,\text{ for \,} b \in \R\, .$
\end{remark}
\begin{proof}
Note that
\[
\sum_i \exp_t((a_i + b) - \lambda_t(\a + b\, \one)) = \sum_i \exp_t\big(a_i - \underbrace{(\lambda_t(\a + b\, \one) - b)}_{ = \lambda_t(\a)}\big) = 1\,\,\,\, \text{ for \,\,\,} \forall \a \in \R^\nclass\, .
\]
The claim follows immediately.
\end{proof}

\begin{proposition}
The loss function $\Delta_{F_t}\big(\y,\exp_t(\ah-\lambda_t(\ah))$ for\, $0 \leq t < 1$\, is convex for
\[\ah \in \{\a' + \R\,\one \mid -\frac{1}{1-t}\leq \a'\}\,.\]
\end{proposition}
\begin{proof} 
Using the definition of the dual function $\Fs^*$ and its derivative $\fs^*$, we can write
\begin{align*}
\Delta_{F_t}(\y, \yh) &= F_t(\y) - F_t(\yh) - (\y - \yh)\cdot f_t (\yh) && \big(\yh = \exp_t(\ah - \lambda_t(\ah)\, \one)\big)\\
& = F_t(\y) - F_t(\yh) - 
    (\y - \yh) \cdot \log_t \circ \exp_t(\ah - \lambda \one)\\
& = F_t(\y) - F_t(\yh) - 
    (\y - \yh) \cdot (\ah - \lambda_t(\ah)\, \one) && 
    \big((\y - \yh) \cdot \mathbf{1} = 1 - 1 = 0\big)\\
    & = \underbrace{F_t(\y) - \y \cdot
    (\fs^*_t)^{-1}(\y)}_{-\Fs^*_t((\fs^*_t)^{-1}(\y))}\, +\, \y \cdot (\fs^*_t)^{-1}(\y)\,
    \underbrace{-\, F_t(\yh) + \yh\cdot\ah}_{\Fs^*_t(\ah)}\, - \,\y\cdot\ah\\
    & = \Fs^*_t(\ah) -\Fs_t^*((\fs^*_t)^{-1}(\y)) - (\ah - (\fs^*_t)^{-1}(\y)) \cdot \y\\
    & = \Delta_{\Fs^*_t}(\ah,(\fs^*_t)^{-1}(\y))\, .
\end{align*}
Note that the transition from the second line to the third line requires that the assumption $-\frac{1}{1-t} \leq \ah$ holds. The dual function $\Fs_t^*$ satisfies
\[
\Fs_t^*(\a + b\, \one) = \lambda_t(\a + b\, \one) + \frac{1}{2-t}\, \sum_i \exp_t\big((a_i + b) - \lambda_t(\a + b\, \one)\big)^{2-t} = \Fs_t^*(\a) + b\, .
\]
Additionally,
\[
\Delta_{\Fs^*_t}(\ah + b\, \one,(\fs^*_t)^{-1}(\y)) = \Fs^*_t(\ah + b\, \one) -\Fs_t^*((\fs^*_t)^{-1}(\y)) - (\ah + b\,\one - (\fs^*_t)^{-1}(\y)) \cdot \y  = \Delta_{\Fs^*_t}(\ah,(\fs^*_t)^{-1}(\y))\, .
\]
The claim follows by considering the range of $\log_t \circ \exp_t$ and the invariance of the Bregman divergence induced by $\Fs_t^*$ along $\R\, \one$.
\end{proof}

\section{Derivatives of Lagrangian and the Bi-tempered Matching Loss}
\label{a:deriv}
The gradient of $\lambda_t(\a)$ w.r.t. $\a$ 
can be calculated by taking the partial derivative of both sides 
of the equality $1= \sum_j \exp_t(a_j - \lambda_t(\a))$  w.r.t. $a_i$:
\begin{align}
\label{eq:lambda_der}
    0&
    = \sum_j \exp_t(a_j - \lambda_t(\a)\big)^t 
       \;\big(\delta_{ij}- \frac{\partial \lambda_t(\a)}{\partial a_i}\big)
    \nonumber\\
    &= \exp_t\big(a_i - \lambda_t(\a)\big)^t 
- \frac{\partial \lambda_t(\a)}{\partial a_i}\,
    \sum_j \exp_t(a_j - \lambda_t(\a)\big)^t,\;
\text{ where $\delta_{ii}=1$ and $\delta_{ij}=0$ for $i\ne j$\,.}
\end{align}
Therefore $\frac{\partial \lambda_t(\a)}{\partial a_i} =
\frac{\exp_t\big(a_i - \lambda_t(\a)\big)^t}{Z_t},$
where $Z_t=\sum_j \exp_t(a_j - \lambda_t(\a))^t$.
We conclude that $\frac{\partial \lambda_t(\a)}{\partial a_i}$
is the ``$t$-escort distribution'' of the distribution
$\frac{\exp (a_i - \lambda_t(\a))}{Z_1}$.

Similarly, the second derivative of $\lambda_t(\a)$ 
can be calculated by repeating the derivation
on~\eqref{eq:lambda_der}:
\[
\frac{\partial^2 \lambda_t(\a)}{\partial a_i\partial a_j} 
= \frac{1}{Z_t}\,\sum_{j'} t\,\exp_t\big(a_{j'} - \lambda_t(\a)\big)^{2t-1}\, \big(\delta_{ij'} 
- \frac{\partial \lambda_t(\a)}{\partial a_i}\big)\big(\delta_{jj'} 
- \frac{\partial \lambda_t(\a)}{\partial a_j}\big).
\]
Although not immediately obvious from the second
derivative, it is easy to show that $\lambda_t(\a)$ is in fact a convex $\a$.
Also the derivative of the loss $\loss(\ah|\, \y)$ w.r.t. $\hat{a}_i$
(expressed in terms of $\y$ and $\yh=\exp_{t_2}(\ah - \lambda_{t_2}(\ah))$) becomes
\begin{align*}	 \frac{\partial \loss}{\partial \hat{a}_i} & = \sum_j \frac{\partial}{\partial \hat{y}_j}\big(y_j \log_{t_1} y_j - y_j \log_{t_1} \hat{y}_j - \frac{1}{2-t_1}\, y_j^{2-t_1} + \frac{1}{2-t_1}\, \hat{y}_j^{2-t_1}\big)\, \frac{\partial \hat{y}_j}{\partial \hat{a}_i}\\
& = \sum_j (\hat{y}_{j} - y_{j})\;
\hat{y}_{j}^{t_2 - t_1} \;
\Big(\delta_{ij} - \frac{\hat{y}_{i}^{t_2}}{\sum_{j'} \hat{y}_{j'}^{t_2}}\Big)\, .
\end{align*}

\section{Proof of Bayes-risk Consistency}
The conditional risk of the multiclass loss $\bm{l}(\ah)$ with $l_i \coloneqq \lb(\ah|\,y=i),\, i\in[\nclass]$ is defined as
\[
R(\bm{\eta}, \bm{l}(\ah)) = \sum_i \eta_i\, l_i\, ,
\]
where $\eta_i \coloneqq \Puk(y = i|\, \x)$. 
\begin{definition}[Bayes-risk Consistency]
A Bayes-risk consistent loss for multiclass classification is the class of loss functions $\ell$ for which $\ah^{\star}$, the minimizer of $R(\bm{\eta}, \bm{l}(\ah))$, satisfies 
\[
\arg\min_i\, \lb(\ah^\star|\,y=i) \subseteq \argmax_i\, \eta_i\, .
\]
\end{definition}
\begin{proof}[\textbf{Proof of Proposition~\ref{prop:bayes}.}]
For the bi-tempered loss, we have
\[
l_i = -\log_{t_1} \exp_{t_2}(\hat{a}_i - \lambda_{t_2}(\ah)) + \frac{1}{2-t_1}\, \sum_j \exp_{t_2}(\hat{a}_j - \lambda_{t_2}(\ah))^{2-t_1}\, .
\]
Note that the second term is repeated for all classes $i\in[\nclass]$. Also,
\[
R(\bm{\eta}, \bm{l}(\ah)) = -\sum_i \eta_i\, \log_{t_1} \exp_{t_2}(\hat{a}_i - \lambda_{t_2}(\ah)) + \frac{1}{2-t_1}\, \sum_i \exp_{t_2}(\hat{a}_i - \lambda_{t_2}(\ah))^{2-t_1}\, .
\]
The minimizer of $R(\bm{\eta}, \bm{l}(\ah))$ satisfies
\[
\eta_i =  \exp_{t_2}(\hat{a}^\star_i - \lambda_{t_2}(\ah^\star))\, .
\]
Since $-\log_{t_1}$ is a monotonically decreasing function for $0 \leq t_1 < 1$, we have
\[
\arg\min_i\, \lb(\ah^\star|\,y=i) = \arg\min_i -\log_{t_1}\exp_{t_2}(\hat{a}^\star_i - \lambda_{t_2}(\ah^\star)) = \arg\max_i \hat{a}^\star_i \subseteq \arg\max_i \eta_i\, .
\]
\end{proof}
\end{document}